\documentclass{article}

\usepackage{microtype}
\usepackage{graphicx}
\usepackage{subfigure}
\usepackage{booktabs} 

\usepackage{hyperref}

\usepackage[accepted]{icml2021}

\usepackage{amsfonts}
\usepackage{amsmath}
\usepackage{amsthm}
\usepackage{pgfplots}
\usepackage{multirow}

\newcommand{\R}{\mathbb{R}}
\newcommand{\N}{\mathbb{N}}

\DeclareMathOperator{\dom}{dom}

\DeclareMathOperator{\interior}{int}
\DeclareMathOperator{\relativeinterior}{relint}

\DeclareMathOperator{\support}{supp}

\newtheorem{proposition}{Proposition}
\newtheorem{lemma}{Lemma}
\newtheorem{corollary}{Corollary}
\newtheorem*{corollary*}{Corollary}

\newtheorem*{theorem*}{Theorem}
\theoremstyle{definition}

\theoremstyle{remark}
\newtheorem{remark}{Remark}

\icmltitlerunning{Moreau-Yosida $f$-divergences}

\begin{document}

\twocolumn[
\icmltitle{Moreau-Yosida $f$-divergences}




\begin{icmlauthorlist}
\icmlauthor{D\'avid Terj\'ek}{renyi}
\end{icmlauthorlist}

\icmlaffiliation{renyi}{Alfr\'ed R\'enyi Institute of Mathematics, Budapest, Hungary}
\icmlcorrespondingauthor{D\'avid Terj\'ek}{dterjek@renyi.hu}

\icmlkeywords{Machine Learning, ICML}

\vskip 0.3in
]

\printAffiliationsAndNotice{}  

\begin{abstract}
Variational representations of $f$-divergences are central to many machine learning algorithms, with Lipschitz constrained variants recently gaining attention. Inspired by this, we define the Moreau-Yosida approximation of $f$-divergences with respect to the Wasserstein-$1$ metric. The corresponding variational formulas provide a generalization of a number of recent results, novel special cases of interest and a relaxation of the hard Lipschitz constraint. Additionally, we prove that the so-called tight variational representation of $f$-divergences can be to be taken over the quotient space of Lipschitz functions, and give a characterization of functions achieving the supremum in the variational representation. On the practical side, we propose an algorithm to calculate the tight convex conjugate of $f$-divergences compatible with automatic differentiation frameworks. As an application of our results, we propose the Moreau-Yosida $f$-GAN, providing an implementation of the variational formulas for the Kullback-Leibler, reverse Kullback-Leibler, $\chi^2$, reverse $\chi^2$, squared Hellinger, Jensen-Shannon, Jeffreys, triangular discrimination and total variation divergences as GANs trained on CIFAR-10, leading to competitive results and a simple solution to the problem of uniqueness of the optimal critic.
\end{abstract}

\section{Introduction}
Variational representations of divergences between probability measures are central to many machine learning algorithms, such as generative adversarial networks \cite{Nowozinetal2016}, mutual information estimation \cite{Belghazietal2018} and maximization \cite{Hjelmetal2019}, and energy-based models \cite{Arbeletal2020}. One class of such measures is the family of $f$-divergences \cite{Csiszar1963, Alietal1966, Csiszar1967}, generalizing the well-known Kullback-Leibler divergence from information theory. Another is the family of optimal transport distances \cite{Villani2008}, including the Wasserstein-$1$ metric. In general, variational representations are supremums of integral formulas taken over sets of functions, such as the Donsker-Varadhan formula \cite{Donskeretal1976} for the Kullback-Leibler divergence or the Kantorovich-Rubinstein formula \cite{Villani2008} for the Wasserstein-$1$ metric. Informally speaking, one can implement \cite{Nowozinetal2016, Arjovskyetal2017} such a formula by constructing a real-valued neural network called the critic (or discriminator) taking samples from the two probability measures as inputs, which is then trained to maximize the integral formula in order to approximate the supremum, resulting in a learned proxy to the actual divergence of said probability measures. Implementing the Kantorovich-Rubinstein formula in such a way involves restricting the Lipschitz constant of the neural network \cite{Gulrajanietal2017, Petzkaetal2018, Miyatoetal2018, Adleretal2018, Terjek2020}, which effectively stabilizes the approximation procedure. Recently, Lipschitz regularization has been incorporated \cite{Farniaetal2018, Zhouetal2019, Ozairetal2019, Songetal2020, Arbeletal2020, Birrelletal2020} into learning algorithms based on variational formulas of divergences other than the Wasserstein-$1$ metric, leading to the same empirical effect and a number of theoretical benefits.

Inspired by this, we study Lipschitz-constrained variational representations of $f$-divergences. We show that existing instances of such variants are special cases of the Moreau-Yosida approximation of $f$-divergences with respect to the Wasserstein-$1$ metric. To any divergence and pair of probability measures corresponds a set of optimal critics, which are exactly those functions which achieve the supremum in the variational representation. An optimal critic corresponding to $f$-divergences is not Lipschitz in general (not even continuous). Since any function represented by a neural network is Lipschitz, when a neural network is trained to approximate such a divergence, its "target", an optimal critic, will never be reached. We show that when the divergence is replaced by its Moreau-Yosida approximation, the corresponding optimal critics are all Lipschitz continuous with uniformly bounded Lipschitz constants, leading to a divergence which is easier to approximate in practice via neural networks. The approximation is parametrized by a pair of real numbers, one of which controls the sharpness of the approximation and the Lipschitz constant of optimal critics. The other controls the behavior of the approximation such that a special case induces a hard Lipschitz constraint in the variational representation, and other values induce only a Lipschitz penalty term. While instances of the former already appeared in the literature, the latter is novel to our paper. A special case reduces to a novel, unconstrained variational representation of the Wasserstein-$1$ metric.

In order to prove these results, we first generalize the so-called tight variational representation of $f$-divergences to be taken over the space of Lipschitz functions or its quotient space, which is the subspace of functions vanishing at an arbitrary, fixed point. The latter leads to optimal critics being unique, having practical benefits. We additionally characterize the functions achieving the supremum in the variational representation. To apply the results, we propose an algorithm compatible with automatic differentiation frameworks to calculate the tight convex conjugate of $f$-divergences which in most cases does not admit a closed form, using Newton's method in the forward pass and implicit differentiation in the backward pass.

Finally, to demonstrate the usefulness of our results, we propose the Moreau-Yosida $f$-GAN, and implement it for the task of generative modeling on CIFAR-10. The experiments show that it is beneficial to use the Moreau-Yosida approximation as a proxy for $f$-divergences, the novel cases of which often outperform the ones with the hard Lipschitz constraint. On the other hand, the representation over the quotient space leads to a simple solution for the problem of uniqueness of the optimal critic.

To summarize, our contributions are
\begin{itemize}
\item a generalization of the tight variational representation of $f$-divergences between probability measures on compact metric spaces along with a characterization of functions achieving the supremum,
\item a practical algorithm to calculate the tight convex conjugate of $f$-divergences compatible with automatic differentiation frameworks,
\item variational formulas for the Moreau-Yosida approximation of $f$-divergences with respect to the Wasserstein-$1$ metric, including a relaxation of the hard Lipshcitz constraint and an unconstrained variational representation of the Wasserstein-$1$ metric, and
\item the Moreau-Yosida $f$-GAN implementing the variational formulas for the Kullback-Leibler, reverse Kullback-Leibler, $\chi^2$, reverse $\chi^2$, squared Hellinger, Jensen-Shannon, Jeffreys, triangular discrimination and total variation divergences as GANs trained on CIFAR-10, leading to competitive performance.
\end{itemize}

\section{Preliminaries}
\subsection{Notations}
Denote the extended reals $\overline{\R}=\R\cup\{\pm\infty\}$, the nonnegative reals $\R_+$, the extended nonnegative reals $\overline{\R}_+=\R_+ \cup {\infty}$. The indicator of a set $A$ is denoted $i_A$ with $i_A(x)=0$ if $x \in A$ and $i_A(x)=\infty$ otherwise. Absolute continuity and singularity of measures is denoted $\ll$ and $\perp$, the Radon-Nikodym derivative of a measure $\mu$ with respect to a nonnegative measure $\nu$ such that $\mu \ll \nu$ by $\frac{d\mu}{d\nu}$, the support of a measure $\mu$ by $\support(\mu)$, a property to hold almost everywhere with respect to a measure $\mu$ by $\mu$-a.e. The relative interior of a subset $A$ of a vector space is denoted $\relativeinterior A$, which for subsets of $\R$ only differs from the interior for singletons whose relative interior is the singleton itself.

\subsection{Convex analysis \cite{Zalinescu2002}}
Given a topological vector space $X$, denote its topological dual by $X^*$, i.e. the set of real-valued continuous linear maps on $X$, which is a topological vector space itself, and the canonical pairing by $\langle \cdot,\cdot \rangle: X \times X^* \to \R$, which is the continuous bilinear map $((x,x^*) \to \langle x, x^* \rangle = x^*(x))$. Given a function $f: X \to \overline{\R}$, the set $\dom f = \left\{ x \in X : f(x) < \infty \right\}$ is the effective domain of $f$. A function $f$ is proper if $\dom f \neq \emptyset$ and $f(x) > -\infty$ for all $x \in X$, otherwise it is improper. For a convex function $f : X \to \overline{\R}$, its convex conjugate is $f^* : X^* \to \overline{\R}$ defined by $f^*(x^*)=\sup_{x \in X}\{\langle x, x^* \rangle - f(x) \}$, and its subdifferential at $x \in X$ is the set $\partial f(x) = \{ x^* \in X^* \ \vert \ \forall \hat{x} \in X : \langle \hat{x} - x , x^* \rangle \leq f(\hat{x}) - f(x) \}$. The biconjugate $f^{**}$ of $f$ is the conjugate of its conjugate $f^*$, i.e. $f^{**}(x) = \sup_{x^* \in X^*}\{\langle x, x^* \rangle - f^*(x^*) \}$, which is equivalent to $f$ if $f$ is proper, convex and lower semicontinuous. In that case, the supremum of the biconjugate representation is achieved precisely at elements of $\partial f(x)$. Conversely, the supremum in the conjugate representation of $f^*(x^*)$ is achieved at elements of $\partial f^*(x^*) = \{ x \in X \ \vert \ \forall \hat{x}^* \in X^* : \langle x , \hat{x}^* - x^* \rangle \leq f^*(\hat{x}^*) - f^*(x^*) \}$.

\subsection{$f$-divergences}
Given a proper, convex and lower semicontinuous function\footnote{Originally, $f$ is used in place of $\phi$ (hence the name), but we reserve the symbol $f$ for other functions.} $\phi : \R \to \overline{\R}$, a measure $\mu$ and a nonnegative measure $\nu$ on a measurable space $X$, the $f$-divergence $D_\phi(\mu \Vert \nu)$ of $\mu$ from $\nu$ is defined \cite{Csiszar1963, Alietal1966, Csiszar1967, Borweinetal1993, Csiszaretal1999} as
\begin{equation}
\int \phi \circ \frac{d\mu_c}{d\nu} d\nu + \phi'(\infty)\mu_s^+(X) - \phi'(-\infty)\mu_s^-(X).
\end{equation}
Here, $\mu_c \ll \nu, \mu_s \perp \nu$ are the absolutely continuous and singular parts of the Lebesgue decomposition of $\mu$ with respect to $\nu$, $\mu_s^+, \mu_s^- \geq 0$ is the Jordan decomposition of the singular part, and $\phi'(\pm\infty) = \lim_{x \to \pm\infty}{\frac{\phi(x)}{x}} \in \overline{\R}$. The well-known variational representation
\begin{equation}
D_\phi(\mu \Vert \nu)=\sup_{f : X \to \R}\left\{\int f d\mu - \int \phi^* \circ f d\nu\right\}
\end{equation}
can be obtained as the biconjugate of the mapping $(\mu \to D_\phi(\mu \Vert \nu))$. The so-called tight variational representation
\begin{multline}
D_\phi(\mu \Vert \nu)=\sup_{f : X \to \R}\left\{ \int f d\mu 
\right.\\\left.
- \inf_{\sup f(X) - \phi'(\infty) \leq \gamma} \left\{\int \phi_+^* \circ (f - \gamma) d\nu + \gamma \right\} \right\}
\end{multline}
with $\phi_+ = \phi + i_{\R_+}$ was obtained in \citet{Agrawaletal2020} as the biconjugate of the mapping $(\mu \to D_\phi(\mu \Vert \nu) + i_{P(X)}(\mu))$ (already considered in \citet{Rudermanetal2012}), and is valid for pairs of probability measures $\mu,\nu$.

\subsection{Wasserstein-$1$ distance \cite{Villani2008}}
Given probability measures $\mu,\nu$ on a metric space $(X,d)$, the Wasserstein-$1$ distance of $\mu$ and $\nu$ is defined as
\begin{equation}
W_1(\mu,\nu)=\inf_{\pi \in \Pi(\mu,\nu)}{\int d(x_1,x_2) d\pi(x_1,x_2)},
\end{equation}
where $\Pi(\mu,\nu)$ is the set of probability measures supported on the product space $X \times X$ with marginals $\mu$ and $\nu$. It has a well-known variational representation called the Kantorovich-Rubinstein formula
\begin{equation}
W_1(\mu,\nu)=\sup_{\Vert f \Vert_L \leq 1}{\left\{ \int f d\mu - \int f d\nu \right\}},
\end{equation}
where
\begin{equation}
\Vert f \Vert_L = \sup_{x, y \in X, x \neq y}{\left\{\frac{\vert f(x) - f(y) \vert}{d(x, y)}\right\}}
\end{equation}
is the Lipschitz norm of $f$. The supremum is achieved by the so-called Kantorovich potentials $f : X \to \R$, unique $\mu,\nu$-a.e. up to an additive constant.

\subsection{Moreau-Yosida approximation}
Let $(X, d)$ be a metric space and $f: X \to \overline{\R}$ a proper function, and $0 < \lambda, \alpha \in \R$ constants. The Moreau-Yosida approximation of index $\lambda$ and order $\alpha$ of $f$ is defined \cite{Jostetal2008, Dalmaso1993} as 
\begin{equation}
f_{\lambda,\alpha}(x)=\inf_{y \in X}{\{ f(y) + \lambda d(x, y)^\alpha \}}.
\end{equation}
It holds that $\overline{f}(x) = \sup_{\lambda>0}{f_{\lambda,\alpha}(x)} = \lim_{\lambda \to \infty}{f_{\lambda,\alpha}(x)}$, where $\overline{f}$ is the greatest lower semicontinuous function with $\overline{f} \leq f$.

\section{Lipschitz representation of $f$-divergences}
In this work, we consider the set $P(X)$ of probability measures on a compact metric space $(X,d)$, which is itself a compact metric space with the metric $W_1$, metrizing the weak convergence of measures. We prove that the tight variational representation of $D_\phi$ from \citet{Agrawaletal2020} can be generalized in the sense that the supremum can be taken over the set $Lip(X,x_0)$ of Lipschitz continuous functions on $X$ that vanish at an arbitrary base point $x_0 \in X$. This is a strictly smaller set than the set of bounded and measurable functions over which the supremum was taken originally. To apply convex analytic techniques, we consider the duality between vector spaces of measures and Lipschitz functions. This aspect is detailed in Appendix~\ref{appendix_background}. An important property of the choice of vector spaces is that the topology on the space of measures generalizes the usual weak convergence of probability measures \citep{Hanin1999}. Proofs and more precise statements of our propositions can be found in Appendix~\ref{appendix_proofs}.
\begin{proposition}
Given probability measures $\mu, \nu \in P(X)$ and a proper, convex and lower semicontinuous function  $\phi : \R \to \overline{\R}$ strictly convex at $1$ with $\phi(1)=0$ and $1 \in \relativeinterior \dom \phi$, the $f$-divergence $D_\phi$ has the equivalent variational representation
\begin{multline} \label{d_phi_var_rep}
D_\phi(\mu\Vert\nu)=
\sup_{f \in Lip(X)}\left\{ \int f d\mu - D_\phi^*(f\Vert\nu) \right\}\\=
\sup_{f \in Lip(X,x_0)}\left\{ \int f d\mu - D_\phi^*(f\Vert\nu) \right\},
\end{multline}
with the tight convex conjugate $D_\phi^*(\cdot\Vert\nu): Lip(X) \to \R$ being
\begin{multline} \label{d_phi_conj}
D_\phi^*(f\Vert\nu)= \sup_{\mu \in P(X)}\left\{ \int f d\mu - D_\phi(\mu\Vert\nu) \right\}\\=
\min_{\sup f(X) - \phi'(\infty) \leq \gamma} \left\{\int \phi_+^* \circ (f - \gamma) d\nu + \gamma \right\}.
\end{multline}
\end{proposition}
The conjugate $D_\phi^*(\cdot\Vert\nu)$ is a topical function \cite{Mohebi2005}, meaning that $D_\phi^*(f+C\Vert\nu)=D_\phi^*(f\Vert\nu)+C$ and $D_\phi^*(f_1\Vert\nu) \geq D_\phi^*(f_2\Vert\nu)$ both hold for $\forall C \in \R$ and $f_1 \geq f_2$. Based on the constant additivity property, the substitution $D_\phi^*(f\Vert\nu)=\int f d\nu + D_\phi^*\left(f- \int f d\nu\Vert\nu\right)$ leads to
\[
\sup_{f \in Lip(X)}\left\{ \int f d\mu - \int f d\nu - D_\phi^*\left(f- \int f d\nu\Vert\nu\right) \right\},
\]
reinterpreting the variational representation of $D_\phi(\mu\Vert\nu)$ as a penalized variant of maximum mean deviation. A closed form expression for $D_\phi^*(\cdot\Vert\nu)$ is available for the Kullback-Leibler divergence with $D_{KL}^*(f\Vert\nu)=\log \int e^f d\nu$.

We call functions $f_*$ for which $D_\phi(\mu\Vert\nu)=\int f_* d\mu - D_\phi^*(f_*\Vert\nu)$ holds, i.e. $f_* \in \partial D_\phi(\mu\Vert\nu)$, \emph{Csisz\'ar potentials} of $\mu,\nu$. This is in analogy with Kantorovich potentials, which are similarly unique $\mu,\nu$-a.e. up to an additive constant. In the second variational representation in \eqref{d_phi_var_rep}, the additive constant is unique since $f(x_0)=0$ must hold. The following result is built on \citet[Theorem~2.10]{Borweinetal1993}.

\begin{proposition} \label{prop_csiszar_potential}
Given probability measures $\mu, \nu \in P(X)$, a function $f_* \in Lip(X)$ is a Csisz\'ar potential of $\mu,\nu$, i.e. $D_\phi(\mu\Vert\nu)=\int f_* d\mu - D_\phi^*(f_*\Vert\nu)$, if and only if there exists $C \in \R$ such that the conditions
\begin{equation}
\sup f_*(X)+C \leq \phi'(\infty),
\end{equation}
\begin{equation}
\frac{d\mu_c}{d\nu}(x) \in \partial \phi_+^*(f_*(x)+C) \ \nu\text{-a.e.}
\end{equation}
and
\begin{equation}
\support(\mu_s) \subset \{ x \in X : f_*(x)+C = \phi'(\infty) \}
\end{equation}
hold. Such $f_*$ are unique $\mu,\nu$-a.e. up to an additive constant.
\end{proposition}

If $\phi$ is of Legendre type \citep{Borweinetal1993}, then $\phi_+$ and $\phi_+^*$ are both continuously differentiable on $\interior \dom \phi_+$ and $\interior \dom \phi_+^*$, respectively, while ${\phi_+^*}'$ is increasing, and invertible where its value is positive with its inverse given by the strictly increasing $\phi_+'$. With these, the second condition is equivalent to
\begin{equation}
f_*(x)+C=\phi_+'\left(\frac{d\mu_c}{d\nu}(x)\right) \ \mu_c\text{-a.e.}
\end{equation}
Informally, this means that $f_*$ is the strictly increasing image of the likelihood ratio. One can then deduce from the Neyman-Pearson lemma \citep{Reidetal2011} that for the binary experiment of discriminating samples from $\mu$ and $\nu$, the statistical test $(x \to \chi_{[\tau,\infty]}(f_*(x)))$ is a most powerful test for any threshold $\tau \in \R$.

Conversely to the above proposition, given $\nu \in P(X)$ and $f \in Lip(X)$, the same conditions characterize the set of $\mu_* \in P(X)$ for which the supremum is achieved in the conjugate representation of $D_\phi^*(\cdot\Vert\nu)$, i.e. $\mu_* \in \partial D_\phi^*(f\Vert\nu)$. Denoting the optimal $\gamma$ in \eqref{d_phi_conj} by $\gamma_{\phi,\nu}(f)$, for any $\mu_* \in P(X)$ satisfying the conditions in Proposition~\ref{prop_csiszar_potential} with $C = -\gamma_{\phi,\nu}(f)$ one has $\mu_* \in \partial D_\phi^*(f\Vert\nu)$. For the Kullback-Leibler divergence, this reduces to the softmax $\mu_*=\frac{1}{\int e^f d\nu} e^f \cdot \nu$. In case $X$ is a finite set, this leads to a family of prediction functions obtained as gradients of $D_\phi^*(f\Vert\nu)$ \cite{Blondeletal2020}. 

\begin{algorithm}[tb]
   \caption{Calculate $\gamma_{\phi,\nu}(f)$ and $\nabla_f \gamma_{\phi,\nu}(f)$}
   \label{gamma_algorithm}
\begin{algorithmic}
   \STATE {\bfseries Input:} $f, \nu \in \R^n$, $\phi : \R \to \overline{\R}$, $0 < \epsilon, \tau \in \R$
   \IF{$\phi'(\infty) < \infty$}
   \STATE $\gamma = \max(f) - \phi'(\infty) + \epsilon$. 
   \ELSE
   \STATE $\gamma = \langle \nu, f \rangle$
   \ENDIF
   \REPEAT
   \STATE $s = \frac{-\langle \nu, (\phi_+^*)'(f - \gamma) \rangle + 1}{\langle \nu, (\phi_+^*)''(f - \gamma) \rangle}$
   \STATE $\gamma = \gamma - s$
   \UNTIL{$\vert s \vert < \tau$}
   \STATE $\nabla_f \gamma=\frac{\nu \odot (\phi_+^*)''(f - \gamma)}{\langle \nu, (\phi_+^*)''(f - \gamma) \rangle}$
\end{algorithmic}
\end{algorithm}

We propose an algorithm for the practical evaluation of $D_\phi^*(\cdot\Vert\nu)$ when no closed form expression is available in the case when the support of $\nu$ is finite\footnote{Such measures are dense in $(P(X),W_1)$.} and $\phi$ is such that $\phi_+^*$ is twice differentiable on $\interior \dom \phi_+^*$ with non-vanishing second derivative. Assuming that $f$ achieves its maximum on the support of $\nu$ and that $\gamma$ achieving the minimum is unique, finding $\gamma$ reduces to a finite dimensional problem, i.e. $f, \nu$ can be considered as elements of $\R^n$ with $n$ being the number of elements of the support of $\nu$. Based on Newton's method and the implicit function theorem, we propose Algorithm~\ref{gamma_algorithm} to calculate $\gamma_{\phi,\nu}(f)$ and its gradient\footnote{$\langle \cdot,\cdot \rangle$ and $\odot$ denote the standard dot product and the elementwise product in $\R^n$.}. Then, the conjugate can be calculated as 
\begin{equation} \label{conjugate_calc}
D_\phi^*(f\Vert\nu)=\langle \nu, \phi_+^* (f - \gamma_{\phi,\nu}(f)) \rangle + \gamma_{\phi,\nu}(f).
\end{equation}
The derivation of the algorithm can be found in Appendix~\ref{appendix_conjugate}, along with the corresponding functions $\phi_+, \phi_+^*$ and their derivatives for the Kullback-Leibler, reverse Kullback-Leibler, $\chi^2$, reverse $\chi^2$, squared Hellinger, Jensen-Shannon, Jeffreys and triangular discrimination divergences. For the Kullback-Leibler divergence, one has the closed form $\gamma_{\phi,\nu}(f) = \log \int e^f d\nu$.

We found that exploiting the constant additivity property by calculating the conjugate as
\begin{equation} \label{conjugate_calc_stable}
D_\phi^*(f\Vert\nu)= D_\phi^*(f-\max(f)\Vert\nu)+\max(f)
\end{equation}
is beneficial to avoid numerical instabilities. This can be seen as a generalization of the log-sum-exp trick.

\section{Moreau-Yosida approximation of $f$-divergences}
Since the mapping $(\mu \to D_\phi(\mu\Vert\nu))$ from the metric space $(P(X),W_1)$ to $\overline{\R}$ is proper and lower semicontinuous, it is an ideal candidate for Moreau-Yosida approximation, for which the infimum is always achieved since $(P(X),W_1)$ is compact if $(X,d)$ is. Given $0 < \lambda, \alpha \in \R$, the Moreau-Yosida approximation of index $\lambda$ and order $\alpha$ of $D_\phi(\cdot\Vert\nu)$ with respect to $W_1$ is therefore defined as
\begin{equation} \label{myfd_primal}
D_{\phi,\lambda,\alpha}(\mu \Vert \nu) = \min_{\xi \in P(X)}{\left\{ D_\phi(\xi \Vert \nu) + \lambda W_1(\mu,\xi)^\alpha \right\}}.
\end{equation}
This is still a divergence in the sense that $D_{\phi,\lambda,\alpha}(\mu\Vert\nu) \geq 0$ with equality if and only if $\mu=\nu$. The original divergence can be recovered as $D_\phi(\mu \Vert \nu) = \sup_{\lambda>0}{D_{\phi,\lambda,\alpha}(\mu \Vert \nu)} = \lim_{\lambda \to \infty}{D_{\phi,\lambda,\alpha}(\mu \Vert \nu)}$ for any $\alpha > 0$. Moreover, for $\alpha \geq 1$, $D_{\phi,\lambda,\alpha}(\cdot \Vert \nu)$ is Lipschitz continuous with respect to $W_1$. If $\alpha=1$, the Lipschitz constant  is exactly $\lambda$. In some cases\footnote{Since the mapping $(\xi \to \lambda W_1(\mu,\xi)^\alpha)$ is neither convex nor concave if $0 < \alpha < 1$, we could not obtain a variational representation via Fenchel-Rockafellar duality in this case.}, variational representations are available.
\begin{proposition}
Given probability measures $\mu, \nu \in P(X)$, $\lambda > 0$, $\alpha \geq 1$ and a proper, convex and lower semicontinuous function  $\phi : \R \to \overline{\R}$ strictly convex at $1$ with $\phi(1)=0$ and $1 \in \relativeinterior \dom \phi$, the divergence $D_{\phi,\lambda,\alpha}(\mu\Vert\nu)$ has the equivalent variational representation
\begin{equation} \label{myfd_alpha1}
\max_{f \in Lip(X,x_0), \Vert f \Vert_L \leq \lambda}\left\{ \int f d\mu
- D_\phi^*(f\Vert\nu) \right\}
\end{equation}
if $\alpha = 1$, and
\begin{multline} \label{myfd_alphagreaterthan1}
\max_{f \in Lip(X,x_0)}\left\{ \int f d\mu
- D_\phi^*(f\Vert\nu)
\right.\\\left.\vphantom{\int f}
- (\alpha-1) \alpha^{\frac{\alpha}{1-\alpha}} \lambda^{\frac{1}{1-\alpha}} \Vert f \Vert_L^{\frac{\alpha}{\alpha-1}} \right\}
\end{multline}
if $\alpha > 1$.
\end{proposition}
In the limit $\alpha \to 1$, \eqref{myfd_alphagreaterthan1} converges to \eqref{myfd_alpha1} in the sense that $\lim_{\alpha \to 1}{(\alpha-1) \alpha^{\frac{\alpha}{1-\alpha}} \lambda^{\frac{1}{1-\alpha}} \Vert f \Vert_L^{\frac{\alpha}{\alpha-1}}} = 0$ if $\Vert f \Vert_L \leq \lambda$ and $\infty$ otherwise, providing an unconstrained relaxation of the hard constraint $\Vert f \Vert_L \leq \lambda$.

Choosing $\phi = i_{\{1\}}$ (so that $D_\phi(\cdot\Vert\nu)=i_{\{\nu\}}$ and $D_\phi^*(f\Vert\nu)=\int f d\nu$), one has $D_{\phi,\lambda,\alpha}(\mu\Vert\nu)=\lambda W_1(\mu,\nu)^\alpha$, leading to the following unconstrained variational representation of $W_1$.
\begin{proposition}\label{unconstrained_wasserstein}
Given $\mu,\nu \in P(X)$, $\lambda > 0$ and $\alpha>1$, $W_1(\mu,\nu)$ has the equivalent unconstrained variational representation
\begin{multline}
\left(\frac{1}{\lambda}
\max_{f \in Lip(X,x_0)}\left\{ \int f d\mu
- \int f d\nu
\right.\right.\\\left.\left.\vphantom{\int f}
- (\alpha-1) \alpha^{\frac{\alpha}{1-\alpha}} \lambda^{\frac{1}{1-\alpha}} \Vert f \Vert_L^{\frac{\alpha}{\alpha-1}} \right\}\right)^{\frac{1}{\alpha}}.
\end{multline}
The maximum is achieved at $\alpha\lambda W_1(\mu,\nu)^{\alpha-1} f_*$, with $f_*$ being a Kantorovich potential of $\mu,\nu$.
\end{proposition}
As stated, subgradients of the mapping $(\mu \to \lambda W_1(\mu,\nu)^\alpha)$ are nothing but the Kantorovich potentials $f_*$ achieving the supremum in the Kantorovich-Rubinstein formula, scaled by the coefficient $\alpha\lambda W_1(\mu,\nu)^{\alpha-1}$. This allows the characterization of subgradients of the mapping $(\mu \to D_{\phi,\lambda,\alpha}(\mu \Vert \nu))$.

\begin{proposition} \label{prop_csiszar_kantorovich_potential}
Given probability measures $\mu, \nu \in P(X)$, $\lambda > 0$, $\alpha \geq 1$ and a proper, convex and lower semicontinuous function  $\phi : \R \to \overline{\R}$ strictly convex at $1$ with $\phi(1)=0$ and $1 \in \relativeinterior \dom \phi$, let $\xi_* \in P(X)$ be a probability measure achieving the minimum in \eqref{myfd_primal}, i.e. for which $D_{\phi,\lambda,\alpha}(\mu \Vert \nu)=D_\phi(\xi_*\Vert\nu) + \lambda W_1(\mu,\xi_*)^\alpha$ holds. Then there exists an $f_* \in Lip(X)$ achieving the maximum in \eqref{myfd_alpha1} if $\alpha=1$ or \eqref{myfd_alphagreaterthan1} if $\alpha>1$, which is a Csisz\'ar potential of $\xi_*,\nu$ and $\alpha\lambda W_1(\mu,\xi_*)^{\alpha-1}$ times a Kantorovich potential of $\mu,\xi_*$ at the same time.
\end{proposition}
These imply that for any $\tau \in \R$, the mapping $(x \to \chi_{[\tau,\infty]}(f_*(x)))$ is a most powerful test for discriminating samples from $\xi_*$ and $\nu$, and that $\Vert f_* \Vert_L = \alpha\lambda W_1(\mu,\xi_*)^{\alpha-1}$. Informally, since $\xi_*$ is close to $\mu$ in $W_1$, the above mapping can be seen as a Lipschitz regularized version of a most powerful test for discriminating $\mu$ and $\nu$.

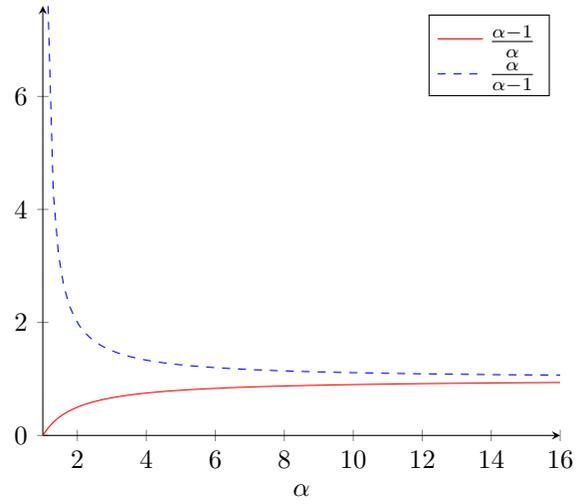
\begin{figure}[ht]
\begin{center}
\resizebox{.95\linewidth}{!}{
\begin{tikzpicture}
\begin{axis}[
    axis lines = left,
    xlabel = $\alpha$,
]
\addplot [
    domain=1:16, 
    samples=100, 
    color=red,
]
{(x-1)/x};
\addlegendentry{$\frac{\alpha-1}{\alpha}$}
\addplot [
    domain=1:16, 
    samples=100, 
    color=blue,
    dashed
]
{x/(x-1)};
\addlegendentry{$\frac{\alpha}{\alpha-1}$}

\end{axis}
\end{tikzpicture}
}
\caption{Multiplier and exponent of $\Vert f \Vert_L$}
\label{alphaplot}
\end{center}
\end{figure}

Consider the reparametrization $\lambda = \frac{1}{\alpha}\beta^{-\alpha}$, so that \eqref{myfd_alphagreaterthan1} reduces to
\begin{multline} \label{myf_repar}
D_{\phi,\frac{1}{\alpha}\beta^{-\alpha},\alpha}(\mu \Vert \nu) =
\max_{f \in Lip(X,x_0)}\left\{ \int f d\mu
\right.\\\left.\vphantom{\int f}
- D_\phi^*(f\Vert\nu)
- \frac{\alpha-1}{\alpha} \left(\beta \Vert f \Vert_L\right)^{\frac{\alpha}{\alpha-1}} \right\}.
\end{multline}
A plot of the respective values of the multiplier and the exponent for $\beta=1$ and $\alpha \in [1,16]$ are visualized in Figure~\ref{alphaplot}. In the limit $\alpha \to \infty$, the multiplier and exponent both converge to $1$. On the other hand, one has $\lim_{\alpha \to 1} \frac{\alpha-1}{\alpha} \Vert f \Vert_L^{\frac{\alpha}{\alpha-1}} =0$ if $\Vert f \Vert_L \leq 1$ and $\infty$ otherwise.

An interesting special case is the limit $\alpha \to \infty$, resulting in the minimum of $D_\phi(\xi\Vert\nu)$ with $\xi \in P(X)$ ranging over the Wasserstein-$1$ ball of radius $\beta$ centered at $\mu$ as
\begin{multline} \label{myf_repar_alphainf}
\lim_{\alpha \to \infty}{D_{\phi,\frac{1}{\alpha}\beta^{-\alpha},\alpha}(\mu \Vert \nu) } = \min_{\xi \in P(X), W_1(\xi,\mu) \leq \beta}\{ D_\phi(\xi\Vert\nu) \}\\
= \max_{f \in Lip(X,x_0)}\left\{ \int f d\mu
- D_\phi^*(f\Vert\nu)
- \beta \Vert f \Vert_L \right\}.
\end{multline}
This should be contrasted with \eqref{myfd_alpha1} corresponding to the $\alpha=1$ case, which also has a hard constraint, but in the dual formula.

Since the values of the above formulas for a given $f$ are invariant for constant translations $f+C$, the supremums can equivalently be taken over $Lip(X)$ instead of $Lip(X,x_0)$ in all cases.

\section{Moreau-Yosida $f$-GAN}
We propose the Moreau-Yosida $f$-GAN (MY$f$-GAN) as an implementation of the variational formula of the Moreau-Yosida regularization of $D_\phi$ with respect to $W_1$. The function $f$ in \eqref{myfd_alpha1} or \eqref{myfd_alphagreaterthan1} is parametrized by a neural network called the critic, which is trained to maximize the formula inside the maximum, providing an approximation of the exact value of the divergence. One of the measures $\mu,\nu$ is represented by the dataset, and the other by a neural network called the generator. The generator transforms samples from a fixed noise distribution into ones resembling the data distribution, and is trained to minimize the divergence approximated by the critic.

Based on the reparametrized formula \eqref{myf_repar} with the substitution $D_\phi^*(f\Vert\nu)=\int f d\nu + D_\phi^*\left(f- \int f d\nu\Vert\nu\right)$, the two minimax games are the following. First let $\mu$ be the generated distribution and $\nu$ be the data, resulting in the forward ($\rightarrow$) formulation
\begin{multline}
\min_{\theta_g \in \R^l} \max_{\theta_f \in \R^k}
\mathbb{E}_{(\zeta_n, \nu_n) \sim (P_z,P_d)}
\langle g_{\theta_g \#}\zeta_n, f_{\theta_f} \rangle
-\langle \nu_n, f_{\theta_f} \rangle
\\-D_\phi^*(f_{\theta_f}-\langle \nu_n, f_{\theta_f} \rangle \Vert \nu_n)
\\-\frac{\alpha-1}{\alpha}(\beta \Vert f_{\theta_f} \Vert_{L,g_{\hat{\theta}_g \#}\zeta_n,\nu_n})^{\frac{\alpha}{\alpha-1}}.
\end{multline}
Now let $\mu$ be the data and $\nu$ the generated distribution, leading to the reverse ($\leftarrow$) formulation
\begin{multline}
\min_{\theta_g \in \R^l} \max_{\theta_f \in \R^k}
\mathbb{E}_{(\mu_n, \zeta_n) \sim (P_d,P_z)}
\langle \mu_n, f_{\theta_f} \rangle
-\langle g_{\theta_g \#}\zeta_n, f_{\theta_f} \rangle
\\-D_\phi^*(f_{\theta_f}-\langle g_{\hat{\theta}_g \#}\zeta_n, f_{\theta_f} \rangle \Vert g_{\hat{\theta}_g \#}\zeta_n)
\\-\frac{\alpha-1}{\alpha}(\beta \Vert f_{\theta_f} \Vert_{L,\mu_n, g_{\hat{\theta}_g \#}\zeta_n})^{\frac{\alpha}{\alpha-1}}.
\end{multline}

The notation of the minimax games is the following. The functions $f: X \times \R^k \to \R$ and $g: Z \times \R^l \to X$ are the critic and generator neural networks parametrized by weight vectors $\theta_f \in \R^k$ and $\theta_g \in \R^l$, and $f_{\theta_f}, g_{\theta_g}$ are shorthands for $f(\cdot,\theta_f), g(\cdot,\theta_g)$. The latent space is $Z=\R^m$. The sample space $X \subset \R^n$ is a compact subset of Euclidean space equipped with the restriction of the metric induced by the Euclidean norm, e.g. $X=[-1,1]^{3*32*32}$ for CIFAR-10. $P_d \in P(X)$ denotes the data distribution and $P_z \in P(Z)$ the noise distribution, e.g. a standard normal. Empirical measures (corresponding to minibatches) are denoted $\mu_n \sim P$, meaning that $\mu_n = \frac{1}{n} \sum_{i=1}^n \delta_{x_{\mu,i}}$ with $(x_{\mu,i}) \subset X$ being a realization of a sequence of $n$ independent and identical copies of the random variable corresponding to $P$. The empirical measure corresponding to the generated distribution is obtained as the pushforward $g_{\theta_g \#}\zeta_n$ of the latent empirical measure $\zeta_n$ (a minibatch of noise samples) through the generator $g_{\theta_g}$. Empirical means are denoted $\langle \mu_n, f \rangle = \frac{1}{n} \sum_{i=1}^n f(x_{\mu,i})$. The conjugate $D_\phi^*$ is calculated according to \eqref{conjugate_calc} using the stabilization trick \eqref{conjugate_calc_stable}. By $\hat{\theta}_g$ we denote a copy of $\theta_g$, meaning that $\theta_g$ is not optimized to minimize terms containing the copy, i.e. the loss function of the generator is $\pm\langle f_{\theta_f}, g_{\theta_g \#}\zeta_n \rangle$. The term $\Vert f_{\theta_f} \Vert_{L,\mu_n, \nu_n}$ denotes a possibly data-dependent estimate of $\Vert f_{\theta_f} \Vert_L$. The minimax games include the case $\lim_{\alpha \to \infty} \frac{\alpha-1}{\alpha} = \lim_{\alpha \to \infty} \frac{\alpha}{\alpha-1} = 1$.

\begin{table*}[t]
\caption{MY$f$-GAN performance on CIFAR-10}
\label{results_table}
\vskip 0.15in
\begin{center}
\begin{small}
\begin{sc}
\begin{tabular}{llcccccccc}
\toprule
 & &
\multicolumn{2}{c}{$\beta=0$} & 
\multicolumn{2}{c}{$\alpha=1.05,\beta=1$} & \multicolumn{2}{c}{$\alpha=2,\beta=1$} & 
\multicolumn{2}{c}{$\alpha=\infty,\beta=0.5\to0.2$} \\
\cmidrule(r){3-10}
$D_\phi$ & & IS & FID & IS & FID & IS & FID & IS & FID \\
\midrule
\multirow{2}{*}{Kullback-Leibler} & $\rightarrow$ & 
$7.16$ & 
$34.12$ & 
$8.26$ &
$13.22$ & 
$8.33$ & 
$14.83$ &
 &
 \\ & $\leftarrow$ & 
 & 
 & 
$8.20$ &
$13.85$ & 
$8.09$ & 
$13.42$ &
$8.20$ &
$12.51$ \\
\multirow{2}{*}{reverse Kullback-Leibler} & $\rightarrow$ & 
 & 
 & 
$8.33$ &
$\mathbf{12.97}$ & 
$8.30$ & 
$13.27$ &
 &
 \\ & $\leftarrow$ & 
 & 
 & 
$8.34$ &
$13.24$ & 
$8.17$ & 
$13.13$ &
$8.09$ &
$15.26$ \\
\multirow{2}{*}{$\chi^2$} & $\rightarrow$ & 
 & 
 & 
$8.18$ &
$14.17$ & 
$8.26$ & 
$13.36$ &
 &
 \\ & $\leftarrow$ & 
 & 
 & 
$8.37$ &
$13.36$ & 
$8.23$ & 
$12.95$ &
$8.27$ &
$13.46$ \\
\multirow{2}{*}{reverse $\chi^2$} & $\rightarrow$ & 
 & 
 & 
$\mathbf{8.47}$ &
$13.89$ & 
$8.26$ & 
$14.59$ &
 &
 \\ & $\leftarrow$ & 
 & 
 & 
$8.24$ &
$14.04$ & 
$\mathbf{8.45}$ & 
$12.28$ &
$8.11$ &
$14.17$ \\
\multirow{2}{*}{squared Hellinger} & $\rightarrow$ & 
 & 
 & 
$8.03$ &
$16.41$ & 
$8.07$ & 
$16.06$ &
 &
 \\ & $\leftarrow$ & 
 & 
 & 
$8.25$ &
$15.89$ & 
$8.25$ & 
$13.93$ &
$\mathbf{8.52}$ &
$\mathbf{12.18}$ \\
\multirow{2}{*}{Jensen-Shannon} & $\rightarrow$ & 
$\mathbf{7.51}$ & 
$\mathbf{30.17}$ & 
$8.30$ &
$14.49$ & 
$8.34$ & 
$12.71$ &
 &
 \\ & $\leftarrow$ & 
 & 
 & 
$8.04$ &
$16.04$ & 
$8.37$ & 
$\mathbf{11.57}$ &
$8.27$ &
$12.58$ \\
\multirow{2}{*}{Jeffreys} & $\rightarrow$ & 
 & 
 & 
$8.09$ &
$13.99$ & 
$8.21$ & 
$14.46$ &
 &
 \\ & $\leftarrow$ & 
 & 
 & 
$8.25$ &
$13.32$ & 
$8.34$ & 
$13.04$ &
 &
 \\
\multirow{2}{*}{triangular discrimination} & $\rightarrow$ & 
$6.45$ & 
$43.14$ & 
$8.42$ &
$13.54$ & 
$8.08$ & 
$14.68$ &
 &
 \\ & $\leftarrow$ & 
 & 
 & 
$8.15$ &
$14.28$ & 
$8.35$ & 
$12.21$ &
$8.09$ &
$15.13$ \\
\multirow{2}{*}{total variation} & $\rightarrow$ & 
$7.41$ & 
$31.09$ & 
$8.12$ &
$15.44$ & 
$8.28$ & 
$14.61$ &
 &
 \\ & $\leftarrow$ & 
 & 
 & 
$8.08$ &
$13.53$ & 
$8.12$ & 
$13.77$ &
$8.05$ &
$14.60$ \\
trivial & & 
 & 
 &
$8.07$ & 
$15.97$ & 
$8.04$ &
$14.75$ &
$6.67$ &
$36.48$ \\
\bottomrule
\end{tabular}
\end{sc}
\end{small}
\end{center}
\vskip -0.1in
\end{table*}

\emph{Lipschitz norm estimation.} Rademacher's theorem \cite{Weaver2018} states that if $\Vert f \Vert_L < \infty$ for $f : \R^n \to \R$, then $\Vert (x \to \Vert \nabla f(x) \Vert_2) \Vert_\infty = \Vert f \Vert_L$ holds, i.e. that the supremum of the function mapping $x \in \R^n$ to the Euclidean norm of the gradient of $f$ at $x$ is equal to the Lipschitz norm of $f$. Based on this and the gradient penalty of \citet{Gulrajanietal2017}, we propose for $\Vert f_{\theta_f} \Vert_{L,\mu_n, \nu_n}$ the estimator
\begin{equation}
\mathbb{E}_{\upsilon_n \sim \mathcal{U}[0,1)} \max_{x \in \support(\upsilon_n \mu_n + (1-\upsilon_n) \nu_n)} \Vert \nabla f_{\theta_f}(x) \Vert_2
\end{equation}
giving a lower bound to $\Vert f_{\theta_f} \Vert_L$. Here, $\mathcal{U}[0,1)$ is the uniform distribution on $[0,1)$ from which an empirical measure $\upsilon_n = \frac{1}{n} \sum_{i=1}^n \delta_{u_i}$ is drawn, and $u_n \mu_n + (1-u) \nu_n = \frac{1}{n} \sum_{i=1}^n \delta_{u_i x_{\mu,i} + (1-u_i) x_{\nu,i}}$ denotes the corresponding interpolation of $\mu_n$ and $\nu_n$. This clearly biased estimator leaves room for improvement. Constructing an unbiased estimator would require assuming a distribution for the random variable representing the value of the gradient norm of the critic, which we leave for future work.

\begin{figure}[ht]
\begin{center}
\resizebox{.95\linewidth}{!}{
\begin{tikzpicture}
\begin{axis}[
    axis lines = left,
    xlabel = iteration,
    ymin = 0.5,
    ymax = 2.0,
]
\addplot[smooth, green, thick] table [x=Step, y=Value, col sep=comma] {csv/run-kl_reverse_alpha1_ema-tag-summaries_gradient_norms_mean.csv};
\addlegendentry{$\alpha=1$ (GP)}
\addplot[smooth, blue, thick] table [x=Step, y=Value, col sep=comma] {csv/run-kl_reverse_alpha1p05_ema0p9999-tag-summaries_gradient_norms_mean.csv};
\addlegendentry{$\alpha=1.05$}
\addplot[dotted, green, thick] table [x=Step, y=Value, col sep=comma] {csv/run-kl_reverse_alpha1_ema-tag-summaries_gradient_norms_max_1.csv};
\addplot[dotted, green, thick] table [x=Step, y=Value, col sep=comma] {csv/run-kl_reverse_alpha1_ema-tag-summaries_gradient_norms_min_1.csv};
\addplot[dotted, blue, thick] table [x=Step, y=Value, col sep=comma] {csv/run-kl_reverse_alpha1p05_ema0p9999-tag-summaries_gradient_norms_max_1.csv};
\addplot[dotted, blue, thick] table [x=Step, y=Value, col sep=comma] {csv/run-kl_reverse_alpha1p05_ema0p9999-tag-summaries_gradient_norms_min_1.csv};
\end{axis}
\end{tikzpicture}
}
\caption{$\Vert \nabla f(X) \Vert_2$ with relaxed Lipschitz constraint}
\label{alpha1vs1p05}
\end{center}
\end{figure}
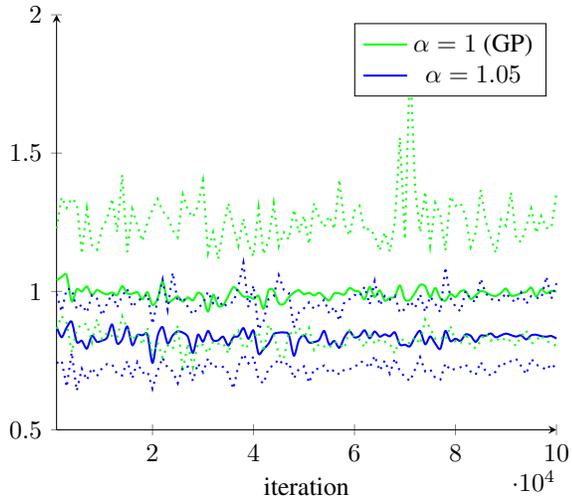

\emph{Relaxation of hard Lipschitz constraint.} We implement the hard constraint case $\alpha=1$ by replacing the last term in the minimax games with the one-sided gradient penalty \cite{Gulrajanietal2017, Petzkaetal2018} $\ell \mathbb{E}_{\upsilon_n \sim \mathcal{U}[0,1)} 
\langle \upsilon_n \mu_n + (1-\upsilon_n) \nu_n, (\max\{0, \Vert \nabla f_{\theta_f}(\cdot) \Vert_2 - \beta^{-1} \})^2 \rangle$ with the coefficient $0 < \ell \in \R$ controlling the strength of the penalty. This is a widely used method to enforce the hard constraint $\Vert f_{\theta_f} \Vert_L \leq \beta^{-1}$. We visualize the maximum, mean and minimum of minibatches of gradient norms of the critic during training in Figure~\ref{alpha1vs1p05} for $\alpha=1$ with the gradient penalty and $\alpha=1.05$ with the estimator detailed above. The $\alpha=1$ case does not enforce the hard constraint, since only the mean of the gradient norms is concentrated around $\beta^{-1}=1$, and not their maximum. The $\alpha=1.05$ case, being a relaxation of the hard constraint, empirically behaves very similarly to an ideal hard constraint implementation, in the sense that the maximum of the gradient norms is concentrated around $\beta^{-1}=1$. This is no surprise in light of Proposition~\ref{prop_csiszar_kantorovich_potential}, since $\Vert f_* \Vert_L = \alpha\lambda W_1(\mu,\xi_*)^{\alpha-1} = \beta^{-\alpha} W_1(\mu,\xi_*)^{\alpha-1} = \beta^{-(1+\epsilon)} W_1(\mu,\xi_*)^{\epsilon}$ is very close to $\beta^{-1}$ in practice for small $\epsilon$, such as $\epsilon=0.05$. We did not observe significant performance differences. This particular experiment used $\ell=10$ and $\phi$ corresponding to the Kullback-Leibler divergence, but we observed identical behavior in other hyperparameter settings as well with a range of $\alpha$ close to $1$. We argue that using the relaxation with some $\alpha=1+\epsilon$ is potentially beneficial for other applications requiring the satisfaction of a hard Lipschitz constraint.

\emph{Choice of $f$-divergence.} Quantitative results in terms of Inception Score (IS) and Fr\'echet Inception Distance (FID) can be seen in Table~\ref{results_table}. Missing values in the unregularized case ($\beta=0$) indicate divergent training, showing that regularization ($\beta>0$) not only improves performance, but leads to convergent training even in cases when it does not seem possible without regularization. The \textsc{trivial} case indicates $D_\phi(\cdot\Vert\nu)=i_{\{\nu\}}$, so that the forward and reverse formulations are identical. In this case, $D_{\phi,\frac{1}{\alpha}\beta^{-\alpha},\alpha}(\mu\Vert\nu)$ reduces to $\frac{1}{\alpha}\beta^{-\alpha}W_1(\mu\Vert\nu)^\alpha$. If $\alpha>1$, this leads to an unconstrained formulation of the Wasserstein GAN corresponding to Proposition~\ref{unconstrained_wasserstein}. The original, constrained Wasserstein GAN with gradient penalty led to an IS of $8.09$ and an FID of $13.40$ in our implementation. This is marginally better than the performance of the unconstrained variant as reported in Table~\ref{results_table}. As shown in Figure~\ref{alpha1vs1p05}, gradient penalty leads to a higher gradient norm than required by the hard constraint, which might lead to the observed marginal performance improvement. Indeed, increasing $\beta$ leads to better performance for the unconstrained variant, e.g. $\beta=0.5$ with $\alpha=2$ led to and IS of $8.14$ and an FID of $13.33$, which is in turn marginally better than the original, constrained variant. While it is hard to tell from these results which $f$-divergence is the best, it is definitely not the \textsc{trivial} one.

\begin{figure}[ht]
\begin{center}
\resizebox{.95\linewidth}{!}{
\begin{tikzpicture}
\begin{axis}[
    axis lines = left,
    xlabel = iteration,
    ymin = -10.0,
    ymax = 10.0,
]
\addplot[smooth, green, thick] table [x=Step, y=Value, col sep=comma] {csv/run-kl_reverse_alpha1p05_ema_noquotient-tag-summaries_d_generated_mean.csv};
\addlegendentry{default}
\addplot[smooth, blue, thick] table [x=Step, y=Value, col sep=comma] {csv/run-kl_reverse_alpha1p05_ema0p9999-tag-summaries_d_generated_mean.csv};
\addlegendentry{quotient}
\addplot[dotted, green, thick] table [x=Step, y=Value, col sep=comma] {csv/run-kl_reverse_alpha1p05_ema_noquotient-tag-summaries_d_generated_max_1.csv};
\addplot[dotted, green, thick] table [x=Step, y=Value, col sep=comma] {csv/run-kl_reverse_alpha1p05_ema_noquotient-tag-summaries_d_generated_min_1.csv};
\addplot[dotted, blue, thick] table [x=Step, y=Value, col sep=comma] {csv/run-kl_reverse_alpha1p05_ema0p9999-tag-summaries_d_generated_max_1.csv};
\addplot[dotted, blue, thick] table [x=Step, y=Value, col sep=comma] {csv/run-kl_reverse_alpha1p05_ema0p9999-tag-summaries_d_generated_min_1.csv};
\end{axis}
\end{tikzpicture}
}
\caption{$f(X)$ for default and quotient critic}
\label{quotientbias}
\end{center}
\end{figure}
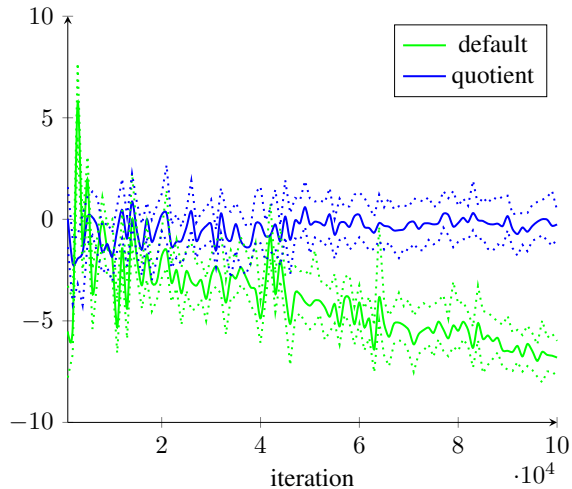

\emph{Quotient critic.} To ensure that $f_{\theta_f} \in Lip(X,x_0)$, we simply modify the forward pass of the critic to return $f_{\theta_f}(x)-f_{\theta_f}(x_0)$ instead of $f_{\theta_f}(x)$. This induces negligible computational overhead since $f_{\theta_f}(x_0)$ can be calculated with a minibatch of size $1$, with the choice of $x_0$ being arbitrary, e.g. the zero vector in our implementation. We call this the quotient critic since $Lip(X,x_0)$ is isomorphic to the quotient space $\frac{Lip(X)}{\R}$. In Figure~\ref{quotientbias} we visualize the maximum, mean and minimum of the critic output over minibatches of generated samples during training. It is clear that the quotient critic solves the drifting of the output of the critic, which was found to hurt performance in some cases \cite{Karrasetal2018, Adleretal2018}. We observed only marginal performance improvement.

\emph{Loss function of the generator.} The reason for picking the penalized mean deviation form of the variational formulas for this application is that in the reverse case, we found that using $-\langle g_{\theta_g \#}\zeta_n, f_{\theta_f} \rangle$ as the loss function of the generator leads to superior performance than using $-D_\phi^*(f_{\theta_f} \Vert g_{\theta_g \#}\zeta_n)$, which cripples performance in most cases. This suggests that gradients of the Csisz\'ar potential $f_*$ might be of greater interest than the gradient of the conjugate $D_\phi^*\left(f_*\Vert\nu\right)$. The latter is a reweighting of the former, since the gradient of the conjugate is a probability distribution, such as the softmax for the Kullback-Leibler divergence.

\emph{Optimal critic has bounded Lipschitz constant.} Notice that while the variational formula of $f$-divergences contains a supremum, the formula of their Moreau-Yosida approximations contains a maximum. This means that in the former case, even if the divergence is finite, the supremum might not be achieved by a Lipschitz function. The variational representation \eqref{d_phi_var_rep} only implies that a sequence of Lipschitz functions converges to a function achieving the supremum, but the limit is not necessarily Lipschitz continuous, in fact it might not even be continuous. On the other hand, for the Moreau-Yosida approximation, the maximum in \eqref{myfd_alpha1} or \eqref{myfd_alphagreaterthan1} is always achieved by a Lipschitz function. Since any neural network is Lipschitz continuous, we argue that a trained critic can provide a better estimate of the Moreau-Yosida approximation, since its target $f_*$ is not only a Csisz\'ar potential of $\xi_*,\nu$ but a scaled Kantorovich potential of $\mu,\xi_*$ as well, implying that it has a bounded Lipschitz constant.

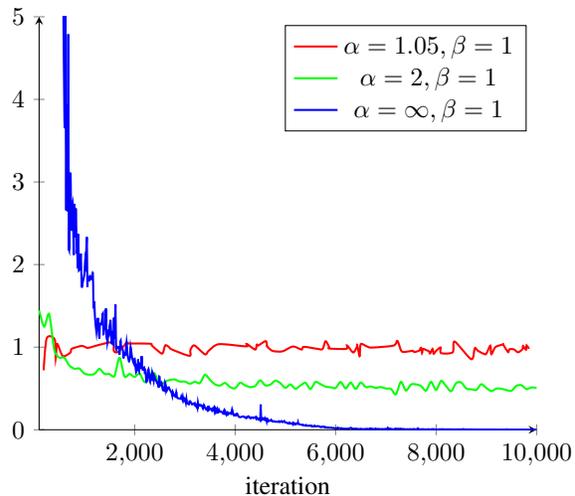
\begin{figure}[ht]
\begin{center}
\resizebox{.95\linewidth}{!}{
\begin{tikzpicture}
\begin{axis}[
    axis lines = left,
    xlabel = iteration,
    ymin = 0.0,
    ymax = 5.0,
    restrict y to domain = -10:10
]
\addplot[smooth, red, thick] table [x=Step, y=Value, col sep=comma] {csv/alpha1p05_max_grad_norms.csv};
\addlegendentry{$\alpha=1.05,\beta=1$}
\addplot[smooth, green, thick] table [x=Step, y=Value, col sep=comma] {csv/alpha2_max_grad_norms.csv};
\addlegendentry{$\alpha=2,\beta=1$}
\addplot[smooth, blue, thick] table [x=Step, y=Value, col sep=comma] {csv/alphainf_max_grad_norms.csv};
\addlegendentry{$\alpha=\infty,\beta=1$}
\end{axis}
\end{tikzpicture}
}
\caption{$\Vert f_{\theta_f} \Vert_{L,\mu_n, \nu_n}$ during training}
\label{max_grad_norm}
\end{center}
\end{figure}

\emph{The $\alpha=2$ and $\alpha=\infty$ cases.} Since $f_*$ is a Kantorovich potential scaled by the coefficient $\beta^{-\alpha} W_1(\mu,\xi_*)^{\alpha-1}$ and the Lipschitz norm of a Kantorovich potential is $1$, the case $\alpha>1$ can be seen as adaptive Lipschitz regularization, with $\Vert f_* \Vert_L$ decaying during training as $\mu$ and $\nu$ drift closer and $W_1(\mu,\xi_*)$ becomes smaller. We visualized $\Vert f_{\theta_f} \Vert_{L,\mu_n, \nu_n}$ in Figure~\ref{max_grad_norm} during training with $\alpha=1.05, 2, \infty$ and $\beta=1$. Ideally, the Lipschitz norm of the critic would vanish. This can be observed in the $\alpha=\infty$ case, which leads to finding a generated distribution with Wasserstein-$1$ distance of $\beta=1$ from the data distribution, accordingly to \eqref{myf_repar_alphainf}. The best FID in Table~\ref{results_table} indicates that it can be beneficial to choose $\alpha=2$ even though the Lipschitz norm does not vanish. While the case $\alpha=\infty$ (where we only consider the case $\leftarrow$) leads to low performance with high values of $\beta$ and unstable training with low values of $\beta$, we found that decaying $\beta$ e.g. from $0.5$ to $0.2$ led to the best IS as can be seen in Table~\ref{results_table}\footnote{Numerical instabilities prevented us from evaluating the Jeffreys divergence in this setting.}. The \textsc{trivial} case does not perform well in this setting, which is not surprising since the exact value of $D_{\phi,\frac{1}{\alpha}\beta^{-\alpha},\alpha}(\mu \Vert \nu)$ is $\infty$ if $\nu$ is not contained in the $W_1$ ball of radius $\beta$ centered at $\mu$, and $0$ otherwise.

Preliminary experiments showed that other values of $\alpha$ behave similarly to the ones we considered, which is why we restricted our attention to the representative values $1.05, 2$ and $\infty$. The implementation was done in TensorFlow, using the residual critic and generator architectures from \citet{Gulrajanietal2017}. Training was done for $100000$ iterations, with $5$ gradient descent step per iteration for the critic, and $1$ for the generator. Additional results, details of the experimental setup and generated images can be found in Appendix~\ref{appendix_experiments}, along with toy examples validating our approach for approximating $f$-divergences through the tight variational representations on categorical and Gaussian distributions. The original $f$-GAN losses \citep{Nowozinetal2016} were particularly unstable in our implementation. Training the critic for $1$ instead of $5$ steps per iteration led to more stability, but even in this case only the $\chi^2$ divergence made it to $100000$ iterations without numerical errors, leading to an IS of $6.49$ and an FID of $40.64$. Source code to reproduce the experiments is available at \url{https://github.com/renyi-ai/moreau-yosida-f-divergences}.

\section{Related work}

In \citet{Farniaetal2018}, $D_{\phi,1,1}$ is defined, and a non-tight variational representation is given for symmetric choices of $D_\phi$. They also prove that $D_{\phi,1,1}$ between the data and generated distributions is a continuous function of the generator parameters, and provide a dual formula for the case $\alpha = 2$ using $W_2$ instead of $W_1$. A future direction is to prove analogous results for general $\alpha,\lambda$ and $W_p$. In \citet{Birrelletal2020}, a generalization of $D_{\phi,1,1}$ is defined with arbitrary IPMs instead of $W_1$, but their assumptions on $\phi$ are more restrictive, and they explicitly define $D_\phi(\mu\Vert\nu)$ to be $\infty$ if $\mu \ll \nu$ does not hold. In \citet{Husainetal2019}, the Lipschitz constrained version of the non-tight variational representation of $D_\phi$ is shown to be a lower bound to the Wasserstein autoencoder objective. In \citet{Laschosetal2019}, it is proved that the supremum in the Donsker-Varadhan formula can equivalently be taken over Lipschitz continuous functions. In \citet{Songetal2020}, based on the non-tight representation, another generalization of $f$-GANs and WGAN is proposed, with the importance weights $r$ analogous to the gradient of $D_\phi^*(f\Vert\nu)$ in our case. Connections to density ratio estimation and sample reweighting are discussed, which apply to our case as well. In \citet{Arbeletal2020}, the Lipschitz constrained version of the Donsker-Varadhan formula is proposed as an objective function for energy-based models. For representation learning by mutual information maximization, \citet{Ozairetal2019} proposes the Lipschitz constrained version of the Donsker-Varadhan formula as a proxy for mutual information, which is shown to be empirically superior to the unconstrained formulation. In \citet{Zhouetal2019}, it is shown that Lipschitz regularization improves the performance of GANs in general other than the Wasserstein GAN. The uniqueness of the optimal critic is investigated, and formulas are proposed for which uniqueness holds. We solve the uniqueness problem in another way, by implementing the quotient critic.

To summarize, the recognition of the primal formula being the Moreau-Yosida regularization of $D_\phi$ with respect to $W_1$ and the case $\alpha \neq 1$ are novel to our paper. This includes the unconstrained variational formula for $W_1$. Regarding $f$-divergences, the tight variational representation over the quotient space $Lip(X,x_0)$ and the characterization of Csisz\'ar potentials are new as well. Additionally, we allow the same generality in terms of the choice of $\phi$ as \citet{Agrawaletal2020}. On the practical side, we proposed an algorithm to calculate the tight conjugate $D_\phi^*\left(f\Vert\nu\right)$ and its gradient. Experimentally, implementations are provided for GANs based on the tight variational representation not only of the Kullback-Leibler divergence, but the reverse Kullback-Leibler, $\chi^2$, reverse $\chi^2$, squared Hellinger, Jensen-Shannon, Jeffreys, triangular discrimination and total variation divergences as well.

\section{Conclusions}
In this paper, we studied the Moreau-Yosida regularization of $f$-divergences with respect to the Wasserstein-$1$ metric in a convex duality framework. We presented variational formulas and characterizations of optimal variables, generalizing a number of existing results and leading to novel special cases of interest, and proposed the MY$f$-GAN as an implementation of the formulas. Future directions include finding the variational formulas for Moreau-Yosida approximation with respect to all Wasserstein-$p$ metrics including the case $0 < \alpha < 1$, improving the estimation of the Lipschitz norm of the critic, making use of the fact that Csisz\'ar-Kantorovich potentials can be seen as Lipschitz-regularized statistical tests, e.g. for sample reweighting, and scaling up to higher-dimensional datasets. Additionally, the results can potentially be applied to learning algorithms other than GANs, such as representation learning by mutual information maximization, energy-based models, generalized prediction functions and density ratio estimation.

\section*{Acknowledgements}

The author was supported by the Hungarian National Excellence Grant 2018-1.2.1-NKP-00008 and by the Hungarian Ministry of Innovation and Technology NRDI Office within the framework of the Artificial Intelligence National Laboratory Program, and would like to thank the Artificial Intelligence Research Group at the Alfr\'ed R\'enyi Institute of Mathematics, especially D\'aniel Varga and Diego Gonz\'alez-S\'anchez, for their generous help.

\bibliography{myfw_icml2021}
\bibliographystyle{icml2021}

\onecolumn

\section{Appendix}

\subsection{Background} \label{appendix_background}
In order to establish the dual formulation of the Moreau-Yosida approximation of $f$-divergences, we will apply techniques from convex analysis, for which we need an appropriate pair of vector spaces that are in duality.

\subsubsection{Functional Analysis}

We recite a number of definitions and results (without proofs) from functional analysis concerning vector spaces of Lipschitz functions, taken from \citet{Cobzasetal2019}.

Let $(X, d)$ be a compact metric space, and denote the $\sigma$-algebra of its Borel subsets by $\mathcal{B}(X)$. A function $\mu : \mathcal{B}(X) \to \R$ is called a $\sigma$-additive measure if $\mu(\bigcup_{i=0}^\infty{A_i}) = \sum_{i=0}^\infty{\mu(A_i)}$ holds for every family $\{A_i : i \in \N\} \subset \mathcal{B}(X)$ of pairwise disjoint elements of $\mathcal{B}(X)$. Any such measure is of bounded variation, i.e. $\vert \mu \vert (X) < \infty$ where 
\begin{equation}
\vert \mu \vert (X) = \sup_{(A_i)_{i \in \{1, \dots, m\}} \text{ is a partition of }X, m \in \N}{\left\{\sum_{i=1}^m{\vert \mu(A_i) \vert }\right\}} 
\end{equation}
is the total variation of $\mu$. Denote by $\mathcal{M}(X)$ the set of $\sigma$-additive measures on $\mathcal{B}(X)$.

A function $f : X \to \R$ is Lipschitz continuous if there exists a number $M \in [0, \infty)$ such that $\vert f(x) - f(y) \vert \leq M d(x, y)$ for all $x, y \in X$. The Lipschitz norm of such an $f$ is defined as
\begin{equation} \label{lip_norm}
\Vert f \Vert_L = \sup_{x, y \in X, x \neq y}{\left\{\frac{\vert f(x) - f(y) \vert}{d(x, y)}\right\}}.
\end{equation}
Denote by $Lip(X)$ the set of Lipschitz continuous functions $f : X \to \R$. Fixing an arbitrary element $x_0 \in X$, the set $Lip(X, x_0) = \{f \in Lip(X) : f(x_0) = 0\}$ is a Banach space with the norm $\Vert . \Vert_L$.
For any $\xi \in \R$, $\mathcal{M}(X, \xi) = \{ \mu \in \mathcal{M}(X) : \mu(X) = \xi\}$ is a vector subspace of $\mathcal{M}(X)$. With the Kantorovich-Rubinstein norm
\begin{equation}
\Vert \mu \Vert_{KR} = \sup_{f \in Lip(X, x_0), \Vert f \Vert_L \leq 1}{\left\{ \int f d\mu \right\}},
\end{equation}
the pair $(\mathcal{M}(X, 0), \Vert . \Vert_{KR})$ is a normed vector space.

\begin{theorem*}
For any $f \in Lip(X, x_0)$ the functional $u_f : \mathcal{M}(X, 0) \to \R$ defined by $u_f(\mu) = \int f d\mu$ is linear and continuous with $\Vert u_f \Vert = \Vert f \Vert_L$. Moreover, every continuous linear functional $v$ on $(\mathcal{M}(X, 0), \Vert . \Vert_{KR})$ is of the form  $v(\mu)=u_f(\mu)$ for a uniquely determined function $f \in Lip(X, x_0)$ with $\Vert v \Vert = \Vert f \Vert_L$. Consequently, the mapping $f \to u_f$ is an isometric isomorphism of $(Lip(X, x_0), \Vert . \Vert_L)$ onto the topological dual $(\mathcal{M}(X, 0), \Vert . \Vert_{KR})^*$, i.e.
\begin{equation}
(Lip(X, x_0), \Vert . \Vert_L) \cong (\mathcal{M}(X, 0), \Vert . \Vert_{KR})^*.
\end{equation}
\end{theorem*}

With the norm
\begin{equation}
\Vert f \Vert_{\max} = \max{\{ \Vert f \Vert_L, \Vert f \Vert_\infty\}},
\end{equation}
the pair $(Lip(X),\Vert . \Vert_{\max})$ is a Banach space. With the Hanin norm
\begin{equation}
\Vert \mu \Vert_H = \inf_{\nu \in \mathcal{M}(X, 0)}{\{ \Vert \nu \Vert_{KR} + \Vert \mu - \nu \Vert_{TV} \}},
\end{equation}
the pair $(\mathcal{M}(X), \Vert . \Vert_H)$ is a normed vector space. The subspace $\mathcal{M}(X,0)$ is closed with respect to the topology generated by $\Vert . \Vert_H$, and the corresponding subspace topology is equivalent to the topology generated by $\Vert . \Vert_{KR}$.
\begin{theorem*}
For any $f \in Lip(X)$ the functional $u_f : \mathcal{M}(X) \to \R$ defined by $u_f(\mu) = \int f d\mu$ is linear and continuous with $\Vert u_f \Vert = \Vert f \Vert_{\max}$. Moreover, every continuous linear functional $v$ on $(\mathcal{M}(X), \Vert . \Vert_H)$ is of the form  $v(\mu)=u_f(\mu)$ for a uniquely determined function $f \in Lip(X)$ with $\Vert v \Vert = \Vert f \Vert_{\max}$. Consequently, the mapping $f \to u_f$ is an isometric isomorphism of $(Lip(X), \Vert . \Vert_{\max})$ onto the topological dual $(\mathcal{M}(X), \Vert . \Vert_H)^*$, i.e.
\begin{equation}
(Lip(X), \Vert . \Vert_{\max}) \cong (\mathcal{M}(X), \Vert . \Vert_H)^*.
\end{equation}
\end{theorem*}
Integration is bilinear, i.e. $\int (\alpha f + \beta g) d\mu = \alpha \int f d\mu + \beta \int g d\mu$ and $\int f d(\alpha \mu + \beta \nu) = \alpha \int f d\mu + \beta \int f d\nu$ for any $\alpha, \beta \in \R$, $f, g \in Lip(X)$ and $\mu, \nu \in \mathcal{M}(X)$.

The set of nonnegative measures is $\mathcal{M}^+(X) = \{ \mu \in \mathcal{M}(X) \ \vert\  \forall A \in \mathcal{B}(X) : \mu(A) \geq 0\}$. The convex set $P(X) = \mathcal{M}(X, 1) \cap \mathcal{M}^+(X)$ is exactly the set of all Borel probability measures on $X$. It is a compact and complete metric space with respect to the Wasserstein-$1$ metric
\begin{equation} \label{wasserstein}
W_1(\mu,\nu)
=
\inf_{\pi \in \Pi(\mu,\nu)}{\int d(x_1,x_2) d\pi(x_1,x_2)},
\end{equation}
where $\Pi(\mu,\nu)$ is the set of probability measures on $X \times X$ with marginals $\mu$ and $\nu$, i.e. given any $A \in \mathcal{B}(X)$, the relations $\pi(A \times X)=\mu(A)$ and $\pi(X \times A)=\nu(A)$ hold.

\begin{theorem*}[Kantorovich-Rubinstein duality]
The metric induced by the norm $\Vert . \Vert_{KR}$ on $P(X)$ is equivalent to the Wasserstein-$1$ metric as
\begin{equation} \label{kantrubduality}
W_1(\mu,\nu) = 
\sup_{f \in Lip(X, x_0), \Vert f \Vert_L \leq 1}{\left\{ \int f d\mu - \int f d\nu \right\}}
= \Vert \mu - \nu \Vert_{KR}
\end{equation}
for any $\mu,\nu \in P(X)$.
\end{theorem*}

\subsubsection{Convex Analysis}
We recite a number of definitions and results (without proofs) from convex analysis on general vector spaces, taken from \citet{Zalinescu2002}.

Let $X,Y$ be separated locally convex topological vector spaces with topological duals $X^*,Y^*$. For an extended real-valued function $f: X \to \overline{\R}$, the function $f^*: X^* \to \overline{\R}$ defined by
\begin{equation} \label{conjugate}
f^*(x^*)=\sup_{x \in X}{\{ \langle x, x^* \rangle -f(x) \}}
\end{equation}
is the (convex) conjugate of $f$, where $\langle \cdot , \cdot \rangle$ is the dual pairing. The conjugate $g^*: X \to \overline{\R}$ of a function $g: X^* \to \overline{\R}$ is defined analogously as
\begin{equation} \label{conjugate_dual}
g^*(x)=\sup_{x^* \in X^*}{\{ \langle x, x^* \rangle -g(x^*) \}},
\end{equation}
leading to the notion of the biconjugate $(f^*)^*=f^{**}$ of $f$, which is the greatest lower semicontinuous convex function with $f^{**} \leq f$.

If $0 < \alpha \in \R$, then
\begin{equation} \label{alphaconjugate}
(\alpha f(\cdot))^*(x^*)=\alpha f^*(\alpha^{-1} x^*),
\end{equation}
if $0 \neq \beta \in \R$, then
\begin{equation} \label{betaconjugate}
(f(\beta \cdot))^*(x^*) = f^*(\beta^{-1} x^*),
\end{equation}
if $x_0 \in X$, then
\begin{equation} \label{translateconjugate}
(f(x_0 + \cdot))^*(x^*) = f^*(x^*) - \langle x_0, x^* \rangle,
\end{equation}
and the Young-Fenchel inequality states that
\begin{equation} \label{youngfenchel}
\forall x \in X, \forall x^* \in X^* : f(x) + f^*(x^*) \geq \langle x, x^* \rangle.
\end{equation}
If $(X, \Vert \cdot \Vert)$ and $(X^*, \Vert \cdot \Vert_*)$ are normed spaces and $f: X \to \overline{\R}$ is defined by $f(x)=\Vert x \Vert$, then
\begin{equation} \label{normconjugate}
f^*(x^*)=
\begin{cases}
0 \text{ if $\Vert x^* \Vert_* \leq 1$,} \\
\infty \text{ otherwise}
\end{cases}
\end{equation}
is the indicator function of the unit ball of the dual space, and if $\psi : \R_+ \to \overline{\R}_+$ is such that $\psi(0)=0$ and $f: X \to \overline{\R}$ is defined by $f(x)=\psi(\Vert x \Vert)$, then
\begin{equation} \label{psiofnormconjugate}
f^*(x^*)=\psi^\#(\Vert x^* \Vert_*),
\end{equation}
where $\psi^\# : \R_+ \to \overline{\R}_+$ is defined by
\begin{equation} \label{psiconj}
\psi^\#(s) = \sup_{0 \leq t \in \R}{\{ st-\psi(t) \}}.
\end{equation}

Given a function $f: X \to \overline{\R}$, the set
\begin{equation}
\dom f = \left\{ x \in X : f(x) < \infty \right\}
\end{equation}
is the effective domain of $f$. A function $f$ is proper if $\dom f \neq \emptyset$ and $f(x) > -\infty$ for all $x \in X$, otherwise it is improper. A function $f$ is convex if
\begin{equation} \label{convex}
\forall x, y \in X, \forall \lambda \in [0,1]: f(\lambda x + (1-\lambda) y) \leq \lambda f(x) + (1-\lambda)f(y)
\end{equation}
holds, and strictly convex if \eqref{convex} holds with $\leq$ replaced by $<$. An $f :\R \to \overline{\R}$ is strictly convex at $x_0 \in \dom f$ if $\lambda f(x_1) + (1 - \lambda) f(x_2) > f(x_0)$ holds for $\forall \lambda \in (0,1)$ and $\forall x_1, x_2 \in \dom f$ such that $\lambda x_1 + (1-\lambda) x_2 = x_0$, unless $x_1 = x_2 = x_0$.

A function $f$ is lower semicontinuous at $x_0 \in X$ if $f(x_0) \leq \liminf_{x \to x_0}{f(x)}$, and $f$ is lower semicontinuous if it is lower semicontinuous at $\forall x_0 \in X$.

\begin{theorem*}[Fenchel-Moreau biconjugation]
Let $X$ be a separated locally convex topological vector space with topological dual $X^*$, and $f : X \to \overline{\R}$ a function. Then $f^{**} \leq f$, and the relation
\begin{equation} \label{biconjugate}
f = f^{**}.
\end{equation}
i.e. $f$ is equivalent to its biconjugate (the conjugate of its conjugate) holds if and only if either $f$ is proper, lower semicontinuous and convex, or $f$ is constant $\pm\infty$.
\end{theorem*}

Given a function $f: X \to \overline{\R}$ and $\hat{x} \in X$, the subdifferential of $f$ at $\hat{x}$ is defined as the set
\begin{equation}
\partial f(\hat{x}) = \{ x^* \in X^* : \forall x \in X : \langle x - \hat{x} , x^* \rangle \leq f(x) - f(\hat{x}) \}.
\end{equation}
Any $x \in \partial f(\hat{x})$ is called a subgradient of $f$ at $\hat{x}$. It is possible that $\partial f(\hat{x}) = \emptyset$, which is always the case if $f(\hat{x}) = \pm \infty$. It holds that if $f$ is proper and $f(x) \in \R$, \eqref{youngfenchel} becomes an equality if and only if $x^* \in \partial f(x)$, or equivalently $x \in \partial f^*(x^*)$. It follows that if $f$ is proper, convex and lower semicontinuous, then
\begin{equation} \label{subgradient}
f(x)=f^{**}(x)=\sup_{x^* \in X^*}{\{ \langle x, x^* \rangle -f^*(x^*) \}}=\langle x, \hat{x}^* \rangle -f^*(\hat{x}^*)
\iff \hat{x}^* \in \partial f(x)
\end{equation}
and
\begin{equation} \label{subgradient_conj}
f^*(x^*)=\sup_{x \in X}{\{ \langle x, x^* \rangle -f(x) \}}=\langle \hat{x}, x^* \rangle -f(\hat{x}) \iff \hat{x} \in \partial f^*(x^*)
\end{equation}
both hold.

The adjoint of a linear map $A : X \to Y$ is the linear map $A^* : Y^* \to X^*$ such that $\langle Ax, y^* \rangle = \langle x, A^* y^* \rangle$ holds for $\forall x \in X, y^* \in Y^*$.

\begin{theorem*}\footnote{This theorem is \citet[Theorem~2.8.3(viii)]{Zalinescu2002} with $Y=\R$.}
Let $A : X \to \R$ be a continuous linear map (so that $A \in X^*$) with adjoint $A^* : \R \to X^*$, and $f : X \to \overline{\R}, g : \R \to \overline{\R}$ be proper convex functions, and consider the proper convex function $h : X \to \overline{\R}$ defined as $h(x)=f(x)+g(Ax)$. If $\dom f \cap A^{-1}(\dom g) \neq \emptyset$ and $0 \in \relativeinterior(A(\dom f)-\dom g)$, then it holds that
\begin{equation} \label{plusunivariate_conjugate}
h^*(x^*) = \min_{\gamma \in \R}{\{ f^*(x^*-A^*\gamma)+g^*(\gamma) \}}
\end{equation}
and
\begin{equation} \label{plusunivariate_subdifferential}
\partial h(x) = \partial f(x) + A^*(\partial g(Ax)).
\end{equation}
\end{theorem*}

\begin{theorem*}[Fenchel-Rockafellar duality]\footnote{This theorem is \citet[Corollary~2.8.5]{Zalinescu2002} and condition \citet[Theorem~2.8.3(iii)]{Zalinescu2002} with $Y=X$, $A:X \to X$ the identity and replacing the dual variable $x^*$ with $-x^*$.}
Let $f,g : X \to \overline{\R}$ be proper convex functions for which $\exists x_0 \in \dom f \cap \dom g$ such that $g$ is continuous at $x_0$. Then it holds that
\begin{equation} \label{fenchelrockafellarduality}
\inf_{x \in X}{\{ f(x) + g(x) \}} = \max_{x^* \in X^*}{\{ -f^*(x^*)-g^*(-x^*)\}}
\end{equation}
and
\begin{equation} \label{fenchelrockafellarduality_min}
\exists \hat{x} \in X : \inf_{x \in X}{\{ f(x) + g(x) \}} = f(\hat{x}) + g(\hat{x}) 
\iff 
\exists \hat{x}^* \in X^* : -\hat{x}^* \in \partial f(\hat{x}) \land \hat{x}^* \in \partial g(\hat{x}).
\end{equation}
\end{theorem*}

\subsubsection{Moreau-Yosida approximation}
Let $(X, d)$ be a metric space and $f: X \to \overline{\R}$ a proper function. Given $0 < \lambda, \alpha \in \R$, the Moreau-Yosida approximation \cite{Dalmaso1993, Jostetal2008} of index $\lambda$ and order $\alpha$ of $f$ is defined as 
\begin{equation}
f_{\lambda,\alpha}(x)=\inf_{y \in X}{\{ f(y) + \lambda d(x, y)^\alpha \}}.
\end{equation}
It holds that
\begin{equation}
\overline{f}(x) = \sup_{\lambda>0}{f_{\lambda,\alpha}(x)} = \lim_{\lambda \to \infty}{f_{\lambda,\alpha}(x)}.
\end{equation}
where $\overline{f}$ is the greatest lower semicontinuous function with $\overline{f} \leq f$.

\begin{theorem*}\label{my_approx_thm}
If $0 < \alpha \leq 1$, then $(f_{\lambda_1,\alpha})_{\lambda_2,\alpha} = f_{\min(\lambda_1,\lambda_2),\alpha}$, and $f_{\lambda,\alpha}$ is the greatest function among those $g \leq f$ for which
\begin{equation}
\forall x, y \in X : \vert g(x) - g(y) \vert \leq \lambda d(x,y)^\alpha
\end{equation}
(i.e. $g$ is Hölder continuous with exponent $\alpha$ and Hölder constant $\lambda$) holds.

The functions $f_{\lambda,1}$ satisfy
\begin{equation}
\forall x, y \in X : \vert f_{\lambda,1}(x)-f_{\lambda,1}(y) \vert \leq \lambda d(x, y)
\end{equation}
(i.e. they are Lipschitz continuous with Lipschitz constant $\lambda$).

If $\alpha \geq 1$, $f$ is non-negative and $0 < r \in \R$, $0 \leq M \in \R$ are constants, then there exists a constant $0 < c(\alpha, \lambda, M, r) \in R$ such that for $\forall z \in X$ it holds that if $f_{\lambda,\alpha}(z) \leq M$, then
\begin{equation}
\forall x, y \in X, d(x,z) \leq r, d(y,z) \leq r :
\vert f_{\lambda,\alpha}(x) - f_{\lambda,\alpha}(y) \vert \leq c d(x,y)
\end{equation}
(i.e. $f_{\lambda,\alpha}$ is locally Lipschitz continuous).
\end{theorem*}

\subsection{Proofs} \label{appendix_proofs}

\begin{proposition}
Given $\nu \in \mathcal{M}^+(X)$ and a proper, convex and lower semicontinuous function  $\phi : \R \to \overline{\R}$ strictly convex at $1$ with $\phi(1)=0$, the function $I_{\phi,\nu} : (\mathcal{M}(X),\Vert.\Vert_H) \to \overline{\R}$ defined by
\begin{equation}
I_{\phi,\nu}(\mu) = 
D_\phi(\mu\Vert\nu)
\end{equation}
is proper, convex and lower semicontinuous, and its conjugate $I_{\phi,\nu}^* : (Lip(X),\Vert.\Vert_{\max}) \to \overline{\R}$ is
\begin{equation}
I_{\phi,\nu}^*(f) = \begin{cases}
\int \phi^* \circ f d\nu \text{ if } f(X) \subseteq [\phi'(-\infty),\phi'(\infty)],\\
\infty \text{ otherwise.}
\end{cases}
\end{equation}
\end{proposition}
\begin{proof}
By \citet[Proposition~4.2.6]{Agrawaletal2020}, one has
\begin{equation}
\sup_{\mu \in \mathcal{M}(X)}\left\{ \int f d\mu - I_{\phi,\nu}(\mu) \right\} = \begin{cases}
\int \phi^* \circ f d\nu \text{ if } f(X) \subseteq [\phi'(-\infty),\phi'(\infty)],\\
\infty \text{ otherwise.}
\end{cases}
\end{equation}
for any bounded and measurable $f : X \to \R$. Any $f \in Lip(X)$ is bounded and measurable, hence the claimed conjugate relation. Clearly $I_{\phi,\nu}$ is convex and proper. For it to be lower semicontinuous, by \eqref{biconjugate} we only need to show that $I_{\phi,\nu}^{**} \geq I_{\phi,\nu}$, i.e. that there exists a sequence $(f_n)$ in $Lip(X)$ such that $\lim_{n \to \infty}{\int f_n d\mu - I_{\phi,\nu}^*(f_n)} \geq I_{\phi,\nu}(\mu)$.

By \citet[Theorem~2.7]{Borweinetal1993} and \eqref{biconjugate}, if $\support(\nu)=X$, then
\begin{equation}
I_{\phi,\nu}(\mu) = \sup_{f \in C(X)}\left\{ \int f d\mu - \int \phi^* \circ f d\nu \right\}
\end{equation}
holds with $C(X)$ being the space of continuous functions on $X$. By \citet[Lemma~3.2.11]{Agrawaletal2020}, the closure of $\dom \phi^*$ is $[\phi'(-\infty),\phi'(\infty)]$, so that for $\int \phi^* \circ f d\nu < \infty$ to hold, $f(X) \subseteq [\phi'(-\infty),\phi'(\infty)]$ is necessary, meaning that the above supremum is equivalent to
\begin{equation}
\sup_{f \in C(X), f(X) \subseteq [\phi'(-\infty),\phi'(\infty)]}\left\{ \int f d\mu - \int \phi^* \circ f d\nu \right\} =
\sup_{f \in C(X)}\left\{ \int f d\mu - I_{\phi,\nu}^*(f) \right\},
\end{equation}
i.e. there exists a sequence $(f_n)$ in $C(X)$ such that $\lim_{n \to \infty}{\int f_n d\mu - I_{\phi,\nu}^*(f_n)} = I_{\phi,\nu}(\mu)$.

In the general case with $\support(\nu) \subset X$, since the support of $\nu$ is closed by definition, being a closed subset of a compact metric space, it is a compact metric space itself with the restriction of the metric $d$. Consider the decomposition $\mu=\mu_1+\mu_2$ defined as $\mu_1(A)=\mu(A \cap \support(\nu))$ and $\mu_2(A)=\mu(A \setminus \support(\nu))$. Then one has
\begin{equation}
D_\phi(\mu\Vert\nu)=D_\phi(\mu_1\Vert\nu)+\phi'(\infty)\mu_2^+(X)-\phi'(-\infty)\mu_2^-(X),
\end{equation}
and by the above considerations there exists a sequence $(f_n)$ in $C(\support(\nu))$ such that $\lim_{n \to \infty}{\int f_n d\mu_1 - I_{\phi,\nu}^*(f_n)} = I_{\phi,\nu}(\mu_1)=D_\phi(\mu_1\Vert\nu)$ and $f_n(X) \subseteq [\phi'(-\infty),\phi'(\infty)]$ for $\forall n$. By the regularity of the measures $\mu_2^+, \mu_2^-$, one has
\begin{equation}
\mu_2^\pm(A) = \sup_{B \in \mathcal{B}(X), B \subseteq A, B \text{ compact}}\left\{\mu_2^\pm(B)\right\}
\end{equation}
for $\forall A \in \mathcal{B}(X)$, i.e. there exist sequences $(B_n^\pm)$ in $\mathcal{B}(X)$ such that $\lim_{n \to \infty} \mu_2^\pm(B_n) = \mu_2^\pm(X)$ with $(B_n^\pm)$ compact, and therefore closed. By the definition of $\mu_2$, we can assume without loss of generality that $B_n^\pm \cap \support(\nu)=\emptyset$ for $\forall n$, and by the definition of the Jordan decomposition, $B_n^+ \cap B_n^- = \emptyset$ can be assumed as well. Define $\tilde{f}_n : \support(\nu) \cup B_n^+ \cup B_n^- \to \R$ as
\begin{equation}
\tilde{f}_n = \begin{cases}
f_n(x) \text{ if } x \in \support(\nu),\\
\beta_n^+ \text{ if } x \in B_n^+,\\
\beta_n^- \text{ if } x \in B_n^-
\end{cases}
\end{equation}
with sequences $(\beta_n^\pm) \subset \R \cap [\phi'(-\infty),\phi'(\infty)]$ such that $\lim_{n\to\infty}\beta_n^\pm = \phi'(\pm\infty)$ and $\forall n : \phi'(-\infty) < \beta_n^\pm < \phi'(\infty)$. Since $\tilde{f}_n$ is clearly continuous, by the Tietze extension theorem, there exists a continuous extension $\hat{f}_n \in C(X)$ agreeing with $\tilde{f}_n$ on $\support(\nu) \cup B_n^+ \cup B_n^-$ with $\hat{f}_n(X) \subseteq [\phi'(-\infty),\phi'(\infty)]$, for which one has
\begin{multline}
\lim_{n \to \infty} \int \hat{f}_n d\mu - I_{\phi,\nu}^*(\hat{f}_n)= 
\lim_{n \to \infty} \int \hat{f}_n d\mu - \int \phi^* \circ \hat{f}_n d\nu \\=
\lim_{n \to \infty} \int f_n d\mu_1 + \int \hat{f}_n d\mu_2 - \int \phi^* \circ f_n d\nu =
D_\phi(\mu_1\Vert\nu) + \lim_{n \to \infty} \int \hat{f}_n d\mu_2 \\=
D_\phi(\mu_1\Vert\nu) + \lim_{n \to \infty} \beta_n^+\mu_2^+(B_n^+)-\beta_n^-\mu_2^-(B_n^-)+\int_{X \setminus B_n^+} \hat{f}_n d\mu_2^+-\int_{X \setminus B_n^-} \hat{f}_n d\mu_2^-.
\end{multline}

If $\phi'(\pm\infty)$ are finite, then $\hat{f}_n$ is bounded uniformly independent of $n$, implying that
\begin{equation}
\lim_{n \to \infty} \int_{X \setminus B_n^+} \hat{f}_n d\mu_2^+-\int_{X \setminus B_n^-} \hat{f}_n d\mu_2^- = 0.
\end{equation}
If one of $\phi'(\pm\infty)$ is infinite, say $\phi'(\infty)$, then
\begin{equation}
\lim_{n \to \infty} \beta_n^+\mu_2^+(B_n^+)-\beta_n^-\mu_2^-(B_n^-)+\int_{X \setminus B_n^+} \hat{f}_n d\mu_2^+-\int_{X \setminus B_n^-} \hat{f}_n d\mu_2^- =  \begin{cases}
\infty \text{ if } \mu_2^+ \neq 0, \\
0 \text{ otherwise},
\end{cases}
\end{equation}
with the case $\phi'(-\infty)=-\infty$ following similarly. If $\phi'(\pm\infty)$ are both infinite, then the above limit is $\infty$ if $\mu_2 \neq 0$ and $0$ otherwise, meaning that in any case
\begin{equation} 
\lim_{n \to \infty} \beta_n^+\mu_2^+(B_n^+)-\beta_n^-\mu_2^-(B_n^-)+\int_{X \setminus B_n^+} \hat{f}_n d\mu_2^+-\int_{X \setminus B_n^-} \hat{f}_n d\mu_2^- = \phi'(\infty)\mu_2^+(X)-\phi'(-\infty)\mu_2^-(X),
\end{equation}
so that one has
\begin{equation}
\lim_{n \to \infty} \int \hat{f}_n d\mu - I_{\phi,\nu}^*(\hat{f}_n)=D_\phi(\mu_1\Vert\nu)+\phi'(\infty)\mu_2^+(X)-\phi'(-\infty)\mu_2^-(X) = D_\phi(\mu\Vert\nu)=I_{\phi,\nu}(\mu).
\end{equation}
This proves that
\begin{equation}
I_{\phi,\nu}(\mu) \leq \sup_{f \in C(X)}\left\{ \int f d\mu - I_{\phi,\nu}^*(f) \right\}
\end{equation}
holds for $\nu$ with $\support(\nu) \subset X$. Since $X$ is compact, by the Stone-Weierstrass theorem $Lip(X)$ is dense in $C(X)$, hence
\begin{equation}
\sup_{f \in C(X)}\left\{ \int f d\mu - I_{\phi,\nu}^*(f) \right\} = \sup_{f \in Lip(X)}\left\{ \int f d\mu - I_{\phi,\nu}^*(f) \right\} = I_{\phi,\nu}^{**}(\mu),
\end{equation}
giving the claim $I_{\phi,\nu}(\mu) \leq I_{\phi,\nu}^{**}(\mu)$.
\end{proof}

We get the non-tight variational representation over $Lip(X)$ as a corollary by \eqref{biconjugate}.
\begin{corollary}
Given $\nu \in \mathcal{M}^+(X)$, $\mu \in \mathcal{M}(X)$ and a proper, convex and lower semicontinuous function  $\phi : \R \to \overline{\R}$ strictly convex at $1$ with $\phi(1)=0$, one has
\begin{equation}
D_\phi(\mu\Vert\nu)=\sup_{f \in Lip(X), f(X) \subseteq [\phi'(-\infty),\phi'(\infty)]}\left\{ \int f d\mu - \int \phi^* \circ f d\nu \right\}.
\end{equation}
\end{corollary}

\begin{proposition}
Given $\nu \in P(X)$ and a proper, convex and lower semicontinuous function  $\phi : \R \to \overline{\R}$ strictly convex at $1$ with $\phi(1)=0$ and $1 \in \relativeinterior \dom \phi$, the function $D_{\phi,\nu} : (\mathcal{M}(X),\Vert.\Vert_H) \to \overline{\R}$ defined by
\begin{equation}
D_{\phi,\nu}(\mu) = 
D_\phi(\mu\Vert\nu) + i_{P(X)}(\mu)
\end{equation}
is proper, convex and lower semicontinuous, and its conjugate $D_{\phi,\nu}^* : (Lip(X),\Vert.\Vert_{\max}) \to \overline{\R}$ is
\begin{equation}
D_{\phi,\nu}^*(f) = 
\min_{\gamma \in \R,\gamma \geq \sup f(X) - \phi'(\infty)} \left\{\int \phi^*_+ \circ (f - \gamma) d\nu + \gamma \right\},
\end{equation}
for which $D_{\phi,\nu}^*(f+C)=D_{\phi,\nu}^*(f)+C$ and $D_{\phi,\nu}^*(f_1) \leq D_{\phi,\nu}^*(f_2)$ holds for $\forall C \in \R$ and $f_1 \leq f_2$, meaning that $D_{\phi,\nu}^*$ is a topical function.
\end{proposition}
\begin{proof}
By \citet[Lemma~4.3.1]{Agrawaletal2020}, one has $D_{\phi,\nu}(\mu)=I_{\phi_+,\nu}(\mu)+i_{\{1\}}(\mu(X))$ with $\phi_+=\phi+i_{\R_+}$. For the constant function $1 \in Lip(X)$, the map $(\mu \to \mu(X) = \langle \mu, 1 \rangle = \int_X 1 d\mu) : \mathcal{M}(X) \to \R$ is linear and continuous, and its adjoint is clearly $(\gamma \to (x \to \gamma)) : \R \to Lip(X)$, mapping constants in $\R$ to the corresponding constant functions in $Lip(X)$. Since $i_{\{ 1 \}} : \R \to \overline{\R}$ is the indicator function of the set $\{ 1 \}$ with its conjugate $i_{\{ 1 \}}^* = (s \to \sup_{t \in \{1\}}{\{ st \}} = s)$ being the identity function, by \eqref{plusunivariate_conjugate} one has
\begin{equation}
D_{\phi,\nu}^*(f) = \min_{\gamma \in \R}{\{ I_{\phi_+,\nu}^*(f - \gamma) + \gamma \}},
\end{equation}
where the minimum can be equivalently taken over those $\gamma \in \R$ for which $I_{\phi_+,\nu}^*(f - \gamma)$ can be finite, i.e. for which $\forall x \in X : \phi_+'(-\infty) \leq f(x) - \gamma \leq \phi_+'(\infty)$ holds for $\forall x \in X$, with the first half being vacuous since $\phi_+'(-\infty)=-\infty$, leading to the claimed conjugate relation. Since $D_{\phi,\nu}$ is the sum of two lower semicontinuous functions, it is itself lower semicontinuous. It is clearly proper and convex as well.

To see that the conditions of \eqref{plusunivariate_conjugate} hold, notice that $\nu \in \dom I_{\phi_+,\nu}$ always holds, while $\dom i_{\{ 1 \}} = \{1\}$ so that $P(X) \subset (\mu \to \mu(X))^{-1} \dom i_{\{ 1 \}} = \{\mu \in \mathcal{M}(X) : \mu(X)=1 \}$, meaning that $\nu \in \dom I_{\phi_+,\nu} \cap (\mu \to \mu(X))^{-1} \dom i_{\{ 1 \}} \neq \emptyset$. Since $1 \in \relativeinterior \dom \phi$ by assumption, either $\dom \phi = \{1\}$, or $\dom \phi$ contains a neighborhood of $1$. In the former case, $\phi = i_{\{1\}}$, and one has $\dom I_{\phi_+,\nu} \cap (\mu \to \mu(X))^{-1} \dom i_{\{ 1 \}} = \{\nu\}$, so that $(\mu \to \mu(X))\dom I_{\phi_+,\nu} - \dom i_{\{ 1 \}} = \{ 1 \} - \{ 1 \} = \{ 0 \}$, for which $\relativeinterior \{ 0 \} = \{ 0 \}$, and the condition holds.

For other choices of $\phi$, there exists $0 < a \in \R$ such that $(1-a,1+a) \subseteq \dom \phi$, and one has for $\forall b \in (1-a,1+a)$ that $I_{\phi_+,\nu}(b\nu)=\int \phi_+(b) d\nu = \phi_+(b) < \infty$, so that $\{b\nu : b \in (1-a,1+a) \} \subseteq \dom I_{\phi_+,\nu}$. This implies that $(1-a,1+a) \subseteq (\mu \to \mu(X))\dom I_{\phi_+,\nu}$, further implying that $(-a,a) \subseteq (\mu \to \mu(X))\dom I_{\phi_+,\nu} - \dom \phi_+$, for which $0 \in \relativeinterior (\mu \to \mu(X))\dom I_{\phi_+,\nu} - \dom \phi_+$ clearly holds, proving that the conditions of \eqref{plusunivariate_conjugate} hold.

The constant additivity property follows from
\begin{multline}
D_{\phi,\nu}^*(f+C)=\sup_{\mu \in \mathcal{M}(X)}{\int (f+C) d\mu - D_{\phi,\nu}(\mu)}=\sup_{\mu \in P(X)}{\int (f+C) d\mu - D_{\phi,\nu}(\mu)}\\=\sup_{\mu \in P(X)}{\int f d\mu + \int C d\mu - D_{\phi,\nu}(\mu)}=\sup_{\mu \in P(X)}{\int f d\mu + C - D_{\phi,\nu}(\mu)}=D_{\phi,\nu}^*(f)+C,
\end{multline}
and the other from
\begin{equation}
f_1 \leq f_2 \implies \int f_1 d\mu - D_{\phi,\nu}(\mu) \leq \int f_2 d\mu - D_{\phi,\nu}(\mu)
\end{equation}
for $\forall \mu \in P(X)$. These properties define topical functions \citep{Mohebi2005}.
\end{proof}

\begin{proposition}
Given $\nu \in P(X)$, $\omega \in \mathcal{M}(X)$ with $\omega(X)=1$ and a proper, convex and lower semicontinuous function  $\phi : \R \to \overline{\R}$ strictly convex at $1$ with $\phi(1)=0$ and $1 \in \relativeinterior \dom \phi$, the function $D_{\phi,\nu,\omega} : (\mathcal{M}(X,0),\Vert.\Vert_{KR}) \to \overline{\R}$ defined by
\begin{equation}
D_{\phi,\nu,\omega}(\mu) = D_\phi(\mu+\omega\Vert\nu)+i_{P(X)}(\mu+\omega)
\end{equation}
is proper, convex and lower semicontinuous, and its conjugate $D_{\phi,\nu,\omega}^* : (Lip(X,x_0),\Vert.\Vert_L) \to \overline{\R}$ is
\begin{equation}
D_{\phi,\nu,\omega}^*(f) = 
\min_{\gamma \in \R,\gamma \geq \sup f(X) - \phi'(\infty)} \left\{\int \phi^*_+ \circ (f - \gamma) d\nu + \gamma \right\} - \int f d\omega.
\end{equation}
\end{proposition}
\begin{proof}
By the previous proposition and \eqref{translateconjugate}, the conjugate relation
\begin{equation}
(\mu \to D_{\phi,\nu}(\mu+\omega))^* = \left(f \to D_{\phi,\nu}^*(f) - \int f d\omega\right)
\end{equation}
holds. Notice that for $D_{\phi,\nu}(\mu+\omega)$ to be finite, $\mu(X)=0$ must hold, since $(\mu+\omega)(X)=1$ is needed, and $\omega(X)=1$ by assumption, hence for $f \in Lip(X)$, one has
\begin{equation}
\sup_{\mu \in \mathcal{M}(X)}\left\{ \int f d\mu - D_{\phi,\nu}(\mu+\omega) \right\} = \sup_{\mu \in \mathcal{M}(X,0)}\left\{ \int f d\mu - D_{\phi,\nu}(\mu+\omega) \right\},
\end{equation}
which is exactly the definition of the value at $f$ of the conjugate of the restriction of $(\mu \to D_{\phi,\nu}(\mu+\omega))$ to $\mathcal{M}(X,0)$. This restriction is clearly proper and convex, and lower semicontinuous as well by being the restriction of a lower semicontinuous function to a closed subspace.
\end{proof}

As a corollary, we get the following dual representation of $D_\phi$ on the space of probability measures, which is the tightest to date, in the sense that the supremum is taken over a set of functions that is a proper subset of those of the previous dual representation \cite{Agrawaletal2020}, which included all bounded and measurable functions. Our representation is over the space of Lipschitz functions vanishing at an arbitrary base point.

\begin{corollary}
Given $\mu,\nu \in P(X)$ and a proper, convex and lower semicontinuous function  $\phi : \R \to \overline{\R}$ strictly convex at $1$ with $\phi(1)=0$ and $1 \in \relativeinterior \dom \phi$, $D_\phi(\mu \Vert \nu)$ has the equivalent variational representation
\begin{equation}
\sup_{f \in Lip(X,x_0)}\left\{ \int f d\mu - \min_{\gamma \in \R,\gamma \geq \sup f(X) - \phi'(\infty)} \left\{\int \phi^*_+ \circ (f - \gamma) d\nu + \gamma \right\} \right\}.
\end{equation}
\end{corollary}
\begin{proof}
Let $\mu \in P(X)$. By \eqref{biconjugate} and the previous proposition,
\begin{multline}
D_{\phi,\nu,\omega}(\mu-\omega) = D_{\phi,\nu,\omega}^{**}(\mu-\omega) \\
= \sup_{f \in Lip(X,x_0)}\left\{ \int f d(\mu-\omega) 
- \min_{\sup f(X) - \phi'(\infty) \leq \gamma} \left\{\int \phi_+^* \circ (f - \gamma) d\nu + \gamma \right\} + \int f d\omega \right\}
\end{multline}
holds, giving the claim.
\end{proof}

\begin{proposition}
Given $\mu,\nu \in P(X)$ and a proper, convex and lower semicontinuous function  $\phi : \R \to \overline{\R}$ strictly convex at $1$ with $\phi(1)=0$ and $1 \in \relativeinterior \dom \phi$, the relation
\begin{equation}
D_\phi(\mu \Vert \nu) = \int f_* d\mu - \min_{\gamma \in \R,\gamma \geq \sup f_*(X) - \phi'(\infty)} \left\{\int \phi^*_+ \circ (f_* - \gamma) d\nu + \gamma \right\}
\end{equation}
holds for $f_* \in Lip(X)$ if and only if there exists $C \in \R$ such that the conditions
\begin{equation}
\sup f_*(X)+C \leq \phi'(\infty),
\end{equation}
\begin{equation}
\frac{d\mu_c}{d\nu}(x) \in \partial \phi_+^*(f_*(x)+C) \text{ almost everywhere with respect to } \nu
\end{equation}
and
\begin{equation}
\support(\mu_s) \subset \{ x \in X : f_*(x)+C = \phi'(\infty) \}
\end{equation}
hold. If $\phi$ is of Legendre type, the second condition is equivalent to
\begin{equation}
f_*(x)+C=\phi_+'\left(\frac{d\mu_c}{d\nu}(x)\right) \text{ almost everywhere with respect to } \mu_c.
\end{equation}
\end{proposition}
\begin{proof}
By \citet[Theorem~2.10]{Borweinetal1993}, given $\mu \in \mathcal{M}(X)$ and $\nu \in \mathcal{M}^+(X)$ with $\support(\nu)=X$, one has
\begin{equation}
D_\phi(\mu\Vert\nu)+\int \phi^* \circ f_* d\nu = \int f_* d\mu
\end{equation}
for $f_* : X \to \R$ continuous if and only if
\begin{equation}
f_*(x)\in[\phi'(-\infty),\phi'(\infty)] \text{ for } \forall x \in X,
\end{equation}
\begin{equation}
\frac{d\mu_c}{d\nu}(x) \in \partial \phi^*(f_*(x)) \text{ almost everywhere with respect to } \nu,
\end{equation}
\begin{equation}
\support \mu_s^- \subset \{ f_*(x) = \phi'(-\infty) \}
\end{equation}
and
\begin{equation}
\support \mu_s^+ \subset \{ f_*(x) = \phi'(\infty) \}
\end{equation}
hold.

In the general case with $\support(\nu) \subset X$, since the support of $\nu$ is closed by definition, it is a compact metric space itself with the restriction of the metric $d$. Consider again the decomposition $\mu=\mu_1+\mu_2$ defined as $\mu_1(A)=\mu(A \cap \support(\nu))$ and $\mu_2(A)=\mu(A \setminus \support(\nu))$. Then one has
\begin{equation}
D_\phi(\mu\Vert\nu)=D_\phi(\mu_1\Vert\nu)+\phi'(\infty)\mu_2^+(X)-\phi'(-\infty)\mu_2^-(X),
\end{equation}
while the optimality conditions above imply that
\begin{equation}
D_\phi(\mu_1\Vert\nu)+\int \phi^* \circ f_* d\nu = \int f_* d\mu_1
\end{equation}
for $f_* : X \to \R$ continuous if and only if
\begin{equation}
f_*(x)\in[\phi'(-\infty),\phi'(\infty)] \text{ for } \forall x \in \support(\nu),
\end{equation}
\begin{equation}
\frac{d\mu_{1c}}{d\nu}(x) \in \partial \phi^*(f_*(x)) \text{ almost everywhere with respect to } \nu,
\end{equation}
\begin{equation}
\support \mu_{1s}^- \subset \{ f_*(x) = \phi'(-\infty) \}
\end{equation}
and
\begin{equation}
\support \mu_{1s}^+ \subset \{ f_*(x) = \phi'(\infty) \}
\end{equation}
hold. For
\begin{equation}
D_\phi(\mu\Vert\nu)+\int \phi^* \circ f_* d\nu = \int f_* d\mu
\end{equation}
to hold, one needs additionally that
\begin{equation}
\phi'(\infty)\mu_2^+(X)-\phi'(-\infty)\mu_2^-(X) = \int f_* d\mu_2,
\end{equation}
which holds exactly if
\begin{equation}
\support \mu_{2s}^- \subset \{ f_*(x) = \phi'(-\infty) \}
\end{equation}
and
\begin{equation}
\support \mu_{2s}^+ \subset \{ f_*(x) = \phi'(\infty) \}
\end{equation}
hold. To summarize, since $\mu_{1c}=\mu_c$ and $\mu_{1s}+\mu_{2s}=\mu_s$, the optimality conditions for $\mu \in \mathcal{M}(X)$ and $\nu \in \mathcal{M}^+(X)$ with $\nu$ not necessarily having full support are the same as cited above from \cite{Borweinetal1993}.

Now consider the tight representation, which follows by taking the conjugate of $(\mu \to D_\phi(\mu\Vert\nu) + i_{P(X)}(\mu) = D_{\phi_+}(\mu\Vert\nu)) + i_{\{1\}}(\int 1 d\mu)$ through \eqref{plusunivariate_conjugate}. Substituting into \eqref{plusunivariate_subdifferential}, one has
\begin{equation}
\partial \left(D_{\phi_+}(\cdot\Vert\nu) + i_{\{1\}}\left(\int 1 d\cdot\right)\right)(\mu) \\
=\partial D_{\phi_+}(\cdot\Vert\nu)(\mu) + (\gamma \to (x \to \gamma))\left(\partial i_{\{1\}}\left(\int 1 d\mu\right)\right),
\end{equation}
which gives the claim since $\partial i_{\{1\}}(1)=\R$, $\phi_+'(-\infty)=-\infty$, $\phi_+'(\infty)=\phi'(\infty)$, $\mu_s^-=0$ and subdifferentials are exactly those $f_*$ for which the supremum is achieved.

If $\phi$ is of Legendre type \citep{Borweinetal1993}, then $\phi_+$ and $\phi_+^*$ are both continuously differentiable on their respective domains, while ${\phi_+^*}'$ is increasing, and invertible on the subset of $\dom \phi_+^*$ where its value is positive with its inverse given by the strictly increasing $\phi_+'$ by \citet[Lemma~2.6]{Borweinetal1993}. With these, the second condition is equivalent to
\begin{equation}
f_*(x)+C=\phi_+'\left(\frac{d\mu_c}{d\nu}(x)\right) \ \mu_c\text{-a.e.}
\end{equation}
\end{proof}

Since we only consider the case of compact sample spaces, the infimum defining the Moreau-Yosida approximation turns into a minimum.

\begin{proposition}
If the metric space $(X,d)$ is compact, the infimum defining the Moreau-Yosida approximation of any $f$-divergence with respect to the Wasserstein-$1$ distance is always achieved as
\begin{equation} \label{infismin}
\inf_{\xi \in P(X)}{\left\{ D_\phi(\xi \Vert \nu) + \lambda W_1(\mu,\xi)^\alpha \right\}} = \min_{\xi \in P(X)}{\left\{ D_\phi(\xi \Vert \nu) + \lambda W_1(\mu,\xi)^\alpha \right\}}
\end{equation}
for any $0 < \lambda, \alpha \in \R$.
\end{proposition}
\begin{proof}
If $(X,d)$ is compact, then $(P(X),W_1)$ is compact as well. Since $(\xi \to D_\phi(\xi \Vert \nu))$ is lower semicontinuous and $(\xi \to W_1(\mu,\xi))$ is continuous, the sum $(\xi \to D_\phi(\xi \Vert \nu) + \lambda W_1(\mu,\xi)^\alpha)$ is lower semicontinuous. The proposition follows from the fact that a lower semicontinuous function on a compact metric space always has a minimum.
\end{proof}

A number of properties follow from the theory of Moreau-Yosida approximation.
\begin{proposition}
For any $0 < \alpha \in \R$, one has $D_\phi(\mu \Vert \nu) = \sup_{\lambda>0}{D_{\phi,\lambda,\alpha}(\mu \Vert \nu)} = \lim_{\lambda \to \infty}{D_{\phi,\lambda,\alpha}(\mu \Vert \nu)}$. Moreover,
\begin{itemize}
\item if $0 < \alpha < 1$, then $D_{\phi,\lambda,\alpha}(\cdot \Vert \nu)$ is H\"older continuous with respect to $W_1$ with exponent $\alpha$ and H\"older constant $\lambda$,
\item if $\alpha=1$, $D_{\phi,\lambda,\alpha}(\cdot \Vert \nu)$ is Lipschitz continuous with respect to $W_1$ with Lipschitz constant $\lambda$, and
\item if $\alpha > 1$, then $D_{\phi,\lambda,\alpha}(\cdot \Vert \nu)$ is locally Lipschitz continuous with respect to $W_1$, hence by $(P(X),W_1)$ being compact $D_{\phi,\lambda,\alpha}(\cdot \Vert \nu)$ is (globally) Lipschitz continuous.
\end{itemize}
\end{proposition}
\begin{proof}
The proposition follows from Theorem~\ref{my_approx_thm}.
\end{proof}

To obtain the dual representations of the Moreau-Yosida approximations of $D_\phi(\cdot\Vert\nu)$ with respect to the Wasserstein-$1$ distance, we need the convex conjugates of the functions mapping probability measures $\xi + \omega$ to $\lambda$ times the $\alpha$th power of their Wasserstein-$1$ distance from a given probability measure $\mu$, which by \eqref{kantrubduality} is equivalent to $\lambda\Vert \xi+\omega-\mu\Vert_{KR}^\alpha$. We consider the cases $0 < \alpha < 1$, $\alpha = 1$ and $\alpha > 1$ separately.

We will need the following lemma.

\begin{lemma}
Let $\psi : \R_+ \to \overline{\R}$ be such that 
\begin{equation}
\psi(t)=\lambda t^\alpha.
\end{equation}
If $1 < \alpha \in \R$, then
\begin{equation} \label{psiconj_alpha1plus}
\psi^\#(s)=(\alpha-1)\alpha^{\frac{\alpha}{1-\alpha}}\lambda^{\frac{1}{1-\alpha}}s^{\frac{\alpha}{\alpha-1}},
\end{equation}
and if $0 < \alpha < 1$, then
\begin{equation} \label{psiconj_alpha01}
\psi^\#(s)=
\begin{cases}
0 \text{ if $s = 0$,}\\
\infty \text{ otherwise.}
\end{cases}
\end{equation}
\end{lemma}
\begin{proof}
By \eqref{psiconj}, $\psi^\#(s)=\sup_{0 \leq t \in \R}{\{ st - \lambda t^\alpha \}}$. Since $\frac{\partial}{\partial{t}} st - \lambda t^\alpha = s - \alpha\lambda t^{\alpha-1}$ and $\frac{\partial}{\partial{t}}^2 st - \lambda t^\alpha = (1-\alpha)\alpha\lambda t^{\alpha-2}$, one has an extremum at $t=\frac{s}{\alpha\lambda}^{\frac{1}{\alpha-1}}$ by letting $\frac{\partial}{\partial{t}}=0$, which is a maximum if $1 < \alpha \in \R$ and a minimum if $0 < \alpha < 1$ by the second derivative test. This implies the proposition.
\end{proof}

\begin{remark} [The case $0 < \alpha < 1$]
By \eqref{psiconj_alpha01} and \eqref{psiofnormconjugate}, it holds that
\begin{equation}
(\xi \to \lambda \Vert \xi + \omega - \mu \Vert_{KR}^\alpha)^* = \left(f \to \begin{cases} 0 \text{ if $\Vert f \Vert_L = 0$},\\ \infty \text{ otherwise} \end{cases}\right),
\end{equation}
which implies that the mapping $(\xi \to \lambda \Vert \xi + \omega - \mu \Vert_{KR}^\alpha)^*$ is not convex by \eqref{biconjugate} (as it is clearly continuous and proper). Hence it is not possible to obtain a dual representation of $\inf_{\xi \in P(X)}{\left\{ D(\xi \Vert \nu) + \lambda W(\mu,\xi)^\alpha \right\}}$ by Fenchel-Rockafellar duality when $0 < \alpha < 1$. Another approach would be to use Toland-Singer duality, which would require the mapping $(\xi \to \lambda \Vert \xi + \omega - \mu \Vert_{KR}^\alpha)^*$ to be concave, but this is also not the case, since it is the composition of a convex and a concave nondecreasing function \citep[Section~3.2.3]{Boydetal2014}.
\end{remark}

\begin{proposition}
Given $\mu, \omega \in P(X)$ and $0 < \lambda \in \R$, let the function $W_{\mu,\omega,\lambda,1}: (\mathcal{M}(X, 0),\Vert.\Vert_{KR}) \to \R$ be defined by
\begin{equation}
W_{\mu,\omega,\lambda,1}(\xi) = \lambda \Vert \xi + \omega - \mu \Vert_{KR}.
\end{equation}
Then, the function $W_{\mu,\lambda,1}$ is proper, convex and continuous, and its convex conjugate $W_{\mu,\omega, \lambda,1}^* : (Lip(X, x_0),\Vert.\Vert_L) \to \overline{\R}$ is
\begin{equation}
W_{\mu,\omega,\lambda,1}^*(f) = 
\begin{cases}
\int f d(\mu-\omega) \text{ if $\Vert f \Vert_L \leq \lambda$,} \\
\infty \text{ otherwise.}
\end{cases}
\end{equation}
\end{proposition}
\begin{proof}
By \eqref{normconjugate},
\begin{equation}
(\xi \to \Vert \xi \Vert_{KR})^* =
\left(f \to
\begin{cases}
0 \text{ if $\Vert f \Vert_L \leq 1$,} \\
\infty \text{ otherwise}
\end{cases}
\right).
\end{equation}
By \eqref{translateconjugate},
\begin{equation}
(\xi \to \Vert \xi + \omega - \mu \Vert_{KR})^* =
\left(f \to
\begin{cases}
-\int f d(\omega - \mu) \text{ if $\Vert f \Vert_L \leq 1$,} \\
\infty \text{ otherwise}
\end{cases}
\right).
\end{equation}
By \eqref{alphaconjugate},
\begin{equation}
(\xi \to \lambda \Vert \xi + \omega - \mu \Vert_{KR})^* 
=
\left(f \to
\begin{cases}
-\lambda \int \lambda^{-1} f d(\omega - \mu) \text{ if $\Vert \lambda^{-1} f \Vert_L \leq 1$,} \\
\infty \text{ otherwise}
\end{cases}
\right),
\end{equation}
which is equivalent to the proposed conjugate.
\end{proof}

\begin{proposition}
Given $\mu, \omega \in P(X)$, $1 < \alpha \in \R$ and $0 < \lambda \in \R$, let the function $W_{\mu,\omega,\lambda,\alpha}: (\mathcal{M}(X, 0),\Vert.\Vert_{KR} \to \R$ be defined by
\begin{equation}
W_{\mu,\omega,\lambda,\alpha}(\xi) = \lambda \Vert \xi + \omega - \mu \Vert_{KR}^\alpha.
\end{equation}
Then, the function $W_{\mu,\omega,\lambda,\alpha}$ is proper, convex and continuous, and its convex conjugate $W_{\mu,\omega,\lambda,\alpha}^* : (Lip(X, x_0),\Vert.\Vert_L) \to \overline{\R}$ is
\begin{equation}
W_{\mu,\omega,\lambda,\alpha}^*(f) = 
\int f d(\mu-\omega) + \alpha^{\frac{\alpha}{1-\alpha}} \lambda^{\frac{1}{1-\alpha}} (\alpha-1) \Vert f \Vert_L^{\frac{\alpha}{\alpha-1}}.
\end{equation}
\end{proposition}
\begin{proof}
By \eqref{psiconj_alpha1plus} and \eqref{psiofnormconjugate}, it holds that
\begin{equation}
(\xi \to \lambda \Vert \xi \Vert_{KR}^\alpha)^* =
\left(f \to
(\alpha-1) \alpha^{\frac{\alpha}{1-\alpha}}\lambda^{\frac{1}{1-\alpha}}  \Vert f \Vert_L^{\frac{\alpha}{\alpha-1}}
\right).
\end{equation}
By \eqref{translateconjugate},
\begin{equation}
(\xi \to \lambda \Vert \xi + \omega - \mu \Vert_{KR}^\alpha)^* =
\left(f \to
(\alpha-1) \alpha^{\frac{\alpha}{1-\alpha}}\lambda^{\frac{1}{1-\alpha}} \Vert f \Vert_L^{\frac{\alpha}{\alpha-1}} - \int f d(\omega - \mu)
\right),
\end{equation}
which is equivalent to the proposed conjugate.
\end{proof}

We obtain the unconstrained variational representation of $W_1$ as a corollary.

\begin{corollary}
Given $\mu, \nu \in P(X)$, $1 < \alpha \in \R$ and $0 < \lambda \in \R$, one has the variational representation
\begin{equation}
\lambda W_1(\mu,\nu)^\alpha = 
\sup_{f \in Lip(X,x_0)}\left\{ \int f d\mu - \int f d\nu - \alpha^{\frac{\alpha}{1-\alpha}} \lambda^{\frac{1}{1-\alpha}} (\alpha-1) \Vert f \Vert_L^{\frac{\alpha}{\alpha-1}} \right\},
\end{equation}
where the supremum is achieved by $\alpha\lambda W_1(\mu,\nu)^{\alpha-1} f_*$ with $f_*$ being a Kantorovich potential of $\mu,\nu$.
\end{corollary}
\begin{proof}
The variational representation follows from the previous proposition and \eqref{biconjugate}. For the other claim, one has
\begin{multline}
\sup_{f \in Lip(X,x_0)}\left\{ \int f d\mu - \int f d\nu - \alpha^{\frac{\alpha}{1-\alpha}} \lambda^{\frac{1}{1-\alpha}} (\alpha-1) \Vert f \Vert_L^{\frac{\alpha}{\alpha-1}} \right\} \\= 
\sup_{\beta \in \R_+, f \in Lip(X,x_0), \Vert f \Vert_L = 1}\left\{ \int \beta f d\mu - \int \beta f d\nu - \alpha^{\frac{\alpha}{1-\alpha}} \lambda^{\frac{1}{1-\alpha}} (\alpha-1) \Vert \beta f \Vert_L^{\frac{\alpha}{\alpha-1}} \right\} \\= 
\sup_{\beta \in \R_+}\left\{ \beta \sup_{f \in Lip(X,x_0), \Vert f \Vert_L = 1}\left\{\int f d\mu - \int f d\nu \right\} - \alpha^{\frac{\alpha}{1-\alpha}} \lambda^{\frac{1}{1-\alpha}} (\alpha-1) \beta^{\frac{\alpha}{\alpha-1}} \right\} \\= 
\sup_{\beta \in \R_+}\left\{ \beta W_1(\mu,\nu) - \alpha^{\frac{\alpha}{1-\alpha}} \lambda^{\frac{1}{1-\alpha}} (\alpha-1) \beta^{\frac{\alpha}{\alpha-1}} \right\}.
\end{multline}
Equating the derivative with respect to $\beta$ to $0$ and solving the resulting equation gives $\beta = \alpha \lambda W_1(\mu,\nu)^{\alpha-1}$.
\end{proof}

Now we have all the information we need in order to invoke Fenchel-Rockafellar duality to obtain the dual representations.

\begin{proposition}
Given $\mu, \nu \in P(X)$, $0 < \lambda \in \R$ and a proper, convex and lower semicontinuous function  $\phi : \R \to \overline{\R}$ strictly convex at $1$ with $\phi(1)=0$ and $1 \in \relativeinterior \dom \phi$, one has
\begin{multline}
\min_{\xi \in P(X)}{\left\{ D_\phi(\xi \Vert \nu) + \lambda W_1(\mu,\xi) \right\}} 
\\=
\max_{f \in Lip(X,x_0), \Vert f \Vert_L \leq \lambda}\left\{ \int f d\mu 
- \min_{\sup f(X) - \phi'(\infty) \leq \gamma} \left\{\int \phi_+^* \circ (f - \gamma) d\nu + \gamma \right\} \right\},
\end{multline}
and for $\xi_* \in P(X)$ such that $\min_{\xi \in P(X)}{\left\{ D_\phi(\xi \Vert \nu) + \lambda W_1(\mu,\xi) \right\}} = D_\phi(\xi_* \Vert \nu) + \lambda W_1(\mu,\xi_*)$, there exists $f_* \in Lip(X,x_0)$ achieving the maximum such that $f_*$ is a Csisz\'ar potential of $\xi_*,\nu$ and $\lambda$ times a Kantorovich potential of $\mu,\xi_*$.
\end{proposition}
\begin{proof}
Substituting $D_{\phi,\nu,\omega}$, $D_{\phi,\nu,\omega}^*$, $W_{\mu,\omega,\lambda,1}$ and $W_{\mu,\omega,\lambda,1}^*$ into \eqref{fenchelrockafellarduality}, for which the condition $\exists \xi_0 \in \dom D_{\phi,\nu,\omega} \cap \dom W_{\mu,\omega,\lambda,1}$ and $W_{\mu,\omega,\lambda,1}$ is continuous at $\xi_0$ clearly holds with $\xi_0 = \nu - \omega$, gives
\begin{equation}
\inf_{\xi \in \mathcal{M}(X,0)}{\left\{ D_{\phi,\nu,\omega}(\xi) + W_{\mu,\omega,\lambda,1}(\xi) \right\}} 
=
\max_{f \in Lip(X,x_0)}{\left\{ -D_{\phi,\nu,\omega}^*(f) - W_{\mu,\omega,\lambda,1}^*(-f) \right\}}.
\end{equation}
Since $W_{\phi,\nu,\omega}(\xi)=\infty$ unless $\xi+\omega \in P(X)$, the infimum is equivalent to
\begin{equation}
\inf_{\xi \in P(X)-\omega}{\left\{ D_\phi(\xi + \omega \Vert \nu) + \lambda \Vert \xi + \omega - \mu \Vert_{KR} \right\}},
\end{equation}
which is further equivalent to
\begin{equation}
\inf_{\xi \in P(X)}{\left\{ D_\phi(\xi \Vert \nu) + \lambda \Vert \xi - \mu \Vert_{KR} \right\}}.
\end{equation}
Since $D_{\mu,\omega,\lambda,1}^*(f)=\infty$ unless $\Vert f \Vert_L \leq \lambda$, the maximum is equivalent to
\begin{equation}
\max_{f \in Lip(X,x_0), \Vert f \Vert_L \leq \lambda}\left\{ 
-\min_{\sup f(X) - \phi'(\infty) \leq \gamma} \left\{\int \phi_+^* \circ (f - \gamma) d\nu + \gamma \right\} 
+ \int f d\omega
-\int -f d(\mu-\omega)
\right\}.
\end{equation}
These, together with \eqref{kantrubduality} and \eqref{infismin} give the claim.

By \eqref{fenchelrockafellarduality_min}, there exists $f_* \in Lip(X,x_0)$ such that $f_*$ is a Csisz\'ar potential of $\xi_*,\nu$ and $\lambda$ times a Kantorovich potential of $\mu,\xi_*$. It achieves the maximum since $\int f_* d\mu - \min_{\sup f_*(X) - \phi'(\infty) \leq \gamma} \left\{\int \phi_+^* \circ (f_* - \gamma) d\nu + \gamma \right\} = \int f_* d\mu -\int f_* d\xi_* + \int f_* d\xi_* - \min_{\sup f_*(X) - \phi'(\infty) \leq \gamma} \left\{\int \phi_+^* \circ (f_* - \gamma) d\nu + \gamma \right\} = \lambda W_1(\mu,\xi_*) + D_\phi(\xi_*\Vert\nu)$.
\end{proof}

\begin{proposition}
Given $\mu, \nu \in P(X)$, $1 < \alpha \in \R$, $0 < \lambda \in \R$ and a proper, convex and lower semicontinuous function  $\phi : \R \to \overline{\R}$ strictly convex at $1$ with $\phi(1)=0$ and $1 \in \relativeinterior \dom \phi$, one has
\begin{multline}
\min_{\xi \in P(X)}{\left\{ D_\phi(\xi \Vert \nu) + \lambda W(\mu,\xi)^\alpha \right\}} 
\\=
\max_{f \in Lip(X,x_0)}\left\{ \int f d\mu
- \min_{\sup f(X) - \phi'(\infty) \leq \gamma} \left\{\int \phi_+^* \circ (f - \gamma) d\nu + \gamma \right\} 
- (\alpha-1) \alpha^{\frac{\alpha}{1-\alpha}} \lambda^{\frac{1}{1-\alpha}} \Vert f \Vert_L^{\frac{\alpha}{\alpha-1}} \right\},
\end{multline}
and for $\xi_* \in P(X)$ such that $\min_{\xi \in P(X)}{\left\{ D_\phi(\xi \Vert \nu) + \lambda W_1(\mu,\xi)^\alpha \right\}} = D_\phi(\xi_* \Vert \nu) + \lambda W_1(\mu,\xi_*)^\alpha$, there exists $f_* \in Lip(X,x_0)$ achieving the maximum such that $f_*$ is a Csisz\'ar potential of $\xi_*,\nu$ and $\alpha\lambda W_1(\mu,\xi_*)^{\alpha-1}$ times a Kantorovich potential of $\mu,\xi_*$.
\end{proposition}
\begin{proof}
Substituting $D_{\phi,\nu,\omega}$, $D_{\phi,\nu,\omega}^*$, $W_{\mu,\omega,\lambda,\alpha}$ and $W_{\mu,\omega,\lambda,\alpha}^*$ into \eqref{fenchelrockafellarduality}, for which the condition $\exists \xi_0 \in \dom D_{\phi,\nu,\omega} \cap \dom W_{\mu,\omega,\lambda,\alpha}$ and $W_{\mu,\omega,\lambda,\alpha}$ is continuous at $\xi_0$ clearly holds with $\xi_0 = \nu - \omega$, gives
\begin{equation}
\inf_{\xi \in \mathcal{M}(X,0)}{\left\{ D_{\phi,\nu,\omega}(\xi) + W_{\mu,\omega,\lambda,\alpha}(\xi) \right\}} 
=
\max_{f \in Lip(X,x_0)}{\left\{ -D_{\phi,\nu,\omega}^*(f) - W_{\mu,\omega,\lambda,\alpha}^*(-f) \right\}}.
\end{equation}
Since $D_{\phi,\nu,\omega}(\xi)=\infty$ unless $\xi+\omega \in P(X)$, the infimum is equivalent to
\begin{equation}
\inf_{\xi \in P(X)-\omega}{\left\{ D_\phi(\xi + \omega \Vert \nu) + \lambda \Vert \xi + \omega - \mu \Vert_{KR}^\alpha \right\}},
\end{equation}
which is further equivalent to
\begin{equation}
\inf_{\xi \in P(X)}{\left\{ D_\phi(\xi \Vert \nu) + \lambda \Vert \xi - \mu \Vert_{KR}^\alpha \right\}}.
\end{equation}
These, together with \eqref{kantrubduality} and \eqref{infismin} give the claim. The function $f_* \in Lip(X,x_0)$ achieving the maximum follows similarly as in the proof of the previous proposition.
\end{proof}

\subsection{Practical evaluation and differentiation of $\gamma_{\phi,\nu}(f)$} \label{appendix_conjugate}

Postponing the general case as future work, we restrict our attention to evaluating $\gamma_{\phi,\nu}(f)$ when the support of $\nu$ is discrete, i.e. there exists $n \in \R$, $\{a_i : 1 \leq i \leq n, \sum_i a_i = 1\} \subset [0,1]$ and $\support{\nu} = \{x_i : 1 \leq i \leq n\} \subset X$ such that $\nu = \sum_i a_i \delta_{x_i}$ can be expressed as a convex combination of Dirac measures. Given $f \in Lip(X,x_0)$, we represent $\nu$ as a vector in the n-dimensional simplex in $\R^n$ defined by $(a_i)$, and $f$ as a vector in $\R^n$ defined by $(f(x_i))$. The minimization problem is then reduced to
\begin{equation}
\min_{\max(f) - \phi'(\infty) \leq \gamma} \left\{\langle \nu, \phi_+^*(f - \gamma) \rangle + \gamma \right\}
\end{equation}
with $\phi_+^*$ being applied element-wise and $\langle \cdot,\cdot \rangle$ being the standard dot product. The first derivative test gives
\begin{equation}
-\langle \nu, (\phi_+^*)'(f - \gamma) \rangle + 1 = 0.
\end{equation}
Assuming the solution is unique, we define $\gamma_{\phi,\nu}(f)$ implicitly as
\begin{equation}
\gamma_{\phi,\nu}(f) = \gamma' \iff 
-\langle \nu, (\phi_+^*)'(f - \gamma') \rangle + 1 = 0.
\end{equation}
Two cases offer closed-form solution. One is the Kullback-Leibler divergence $\phi(x)=x \log x$, for which we get $\gamma_{(x \to x \log x),\nu}(f) = \log \langle \nu, e^f \rangle$, i.e. the term from the Donsker-Varadhan formula. The other is the total variation divergence corresponding to $\phi(x)=\vert x-1 \vert$, for which the mapping
\begin{equation}
\gamma \to \langle \nu, -\chi_{(-\infty,-1)}(f-\gamma)+(f-\gamma)\chi_{[-1,1]}(f-\gamma) \rangle + \gamma
\end{equation}
is nondecreasing for $\gamma \geq \max(f)-1$ by its derivative $\langle \nu, -\chi_{[-1,1]}(f-\gamma) \rangle + 1$ being nonnegative, implying that the optimal value is $\gamma_{(x \to \vert x-1 \vert),\nu}(f)=\max(f)-1$, and the conjugate is $D_{(x \to \vert x-1 \vert)}(f\Vert\nu)=\langle \nu, -\chi_{(-\infty,-1)}(f-\max(f)+1)+(f-\max(f)+1)\chi_{[-1,1]}(f-\max(f)+1) \rangle + \max(f)-1$.

For other choices of $\phi$ considered, it seems that no closed-form solution is available. Instead, we calculate $\gamma_{\phi,\nu}(f)$ by Newton's method, for which we need the derivative of the function whose zeroes define the values of $\gamma_{\phi,\nu}(f)$, which is
\begin{equation}
\langle \nu, (\phi_+^*)''(f - \gamma) \rangle.
\end{equation}
Then, Newton's method suggests that the iteration
\begin{equation}
\gamma_{n+1} = \gamma_n - \frac{-\langle \nu, (\phi_+^*)'(f - \gamma) \rangle + 1}{\langle \nu, (\phi_+^*)''(f - \gamma) \rangle}
\end{equation}
converges to $\gamma_{\phi,\nu}(f)$. For the initial value, the choice $\gamma_0 = \max{f} - \phi'(\infty) + \epsilon$ works for some small $\epsilon > 0$ tuned manually if $\phi'(\infty) < \infty$. The the other cases with $\phi'(\infty) = \infty$, we found the choice $\gamma_0 = \langle \nu, f \rangle$ to work just fine.

To integrate this implicit function into automatic differentiation frameworks, we need to be able to compute the gradients $\nabla_f \gamma_{\phi,\nu}(f)$. Instead of the potentially unstable method of backpropagating through the iterations of Newton's method, we use the implicit function theorem \cite{Tao2016}, which tells us that
\begin{equation}
\nabla_f \gamma_{\phi,\nu}(f) =
\frac{
-\nabla_f(-\langle \nu, (\phi_+^*)'(f - \gamma') \rangle + 1)
}{
\frac{d}{d\gamma} (-\langle \nu, (\phi_+^*)'(f - \gamma') \rangle + 1)
} =
\frac{
\nu \odot (\phi_+^*)''(f - \gamma)
}{
\langle \nu, (\phi_+^*)''(f - \gamma) \rangle
}
\end{equation}
with $\odot$ denoting element-wise product in $\R^n$. Notice that in all cases, $\nabla_f \gamma_{\phi,\nu}(f)$ is in the standard simplex, e.g. for the Kullback-Leibler divergence one has the softmax $\nabla_f \gamma_{\phi,\nu}(f) = \frac{\nu \odot e^f}{\langle \nu, e^f \rangle}$ as the gradient.

We implemented the implicit functions with Newton's method in the forward pass and the backward pass formula given by the implicit function theorem for the Kullback-Leibler, reverse Kullback-Leibler, $\chi^2$, reverse $\chi^2$, squared Hellinger, Jensen-Shannon, Jeffreys and  triangular discrimination divergences. To test the validity of the approach, we optimized an $f \in \R^n$ with gradient descent to maximize the corresponding variational formulas with random categorical distributions $\mu,\nu$ over an alphabet of size $n$, and found that the resulting value for the divergences matched that of the closed-form solution with high accuracy. We found that calculating $\gamma_{\phi,\nu}(f)$ as $\gamma_{\phi,\nu}(f)= \gamma_{\phi,\nu}(f-\max(f))+\max(f)$ is beneficial to avoid numerical instabilities, which can be seen as a generalization of the log-sum-exp trick.  We detail the functions derived from $\phi$ corresponding to the listed $f$-divergences defined by functions of Legendre type below, as well as the corresponding Csisz\'ar potentials.

\subsubsection{Kullback-Leibler divergence}
\begin{equation}
\phi_+(x)= \begin{cases}
x \log(x) - x + 1 \text{ if $x \geq 0$,}\\
\infty \text{ otherwise.}
\end{cases}
\end{equation}
\begin{equation}
\partial \phi_+(x)= \begin{cases}
\left\{ \log(x) \right\} \text{ if $x > 0$,}\\
\emptyset \text{ otherwise.}
\end{cases}
\end{equation}
\begin{equation}
\phi'(\infty)=\infty
\end{equation}
\begin{equation}
\phi_+^*(x)=e^x-1
\end{equation}
\begin{equation}
{\phi_+^*}'(x)=e^x
\end{equation}
\begin{equation}
{\phi_+^*}''(x)=e^x
\end{equation}
\begin{equation}
f_*(x)+C = \log\left(\frac{d\mu_c}{d\nu}(x)\right) \text{ almost everywhere with respect to } \mu_c
\end{equation}

\subsubsection{Reverse Kullback-Leibler divergence}
\begin{equation}
\phi_+(x)= \begin{cases}
x-1-\log(x) \text{ if $x \geq 0$,}\\
\infty \text{ otherwise.}
\end{cases}
\end{equation}
\begin{equation}
\partial \phi_+(x)= \begin{cases}
\left\{ \frac{x-1}{x} \right\} \text{ if $x > 0$,}\\
\emptyset \text{ otherwise.}
\end{cases}
\end{equation}
\begin{equation}
\phi'(\infty)=1
\end{equation}
\begin{equation}
\phi_+^*(x)=\begin{cases}
-\log(1-x) \text{ if $x \leq 1$,}\\
\infty \text{ otherwise.}
\end{cases}
\end{equation}
\begin{equation}
{\phi_+^*}'(x)=\frac{1}{1-x}
\end{equation}
\begin{equation}
{\phi_+^*}''(x)=\frac{1}{(1-x)^2}
\end{equation}
\begin{equation}
f_*(x)+C = \begin{cases}
\frac{\frac{d\mu_c}{d\nu}(x)-1}{\frac{d\mu_c}{d\nu}(x)} \text{ almost everywhere with respect to } \mu_c, \\
1 \text{ if } x \in \support(\mu_s)
\end{cases}
\end{equation}

\subsubsection{$\chi^2$ divergence}
\begin{equation}
\phi_+(x)= \begin{cases}
(x-1)^2 \text{ if $x \geq 0$,}\\
\infty \text{ otherwise.}
\end{cases}
\end{equation}
\begin{equation}
\partial \phi_+(x)= \begin{cases}
\left\{ 2x-2 \right\} \text{ if $x \geq 0$,}\\
\emptyset \text{ otherwise.}
\end{cases}
\end{equation}
\begin{equation}
\phi'(\infty)=\infty
\end{equation}
\begin{equation}
\phi_+^*(x)=
\begin{cases}
\frac{1}{4}x^2+x \text{ if $x \geq -2$},\\
-1 \text{ otherwise.}
\end{cases}
\end{equation}
\begin{equation}
{\phi_+^*}'(x)=
\begin{cases}
\frac{1}{2}x+1 \text{ if $x \geq -2$},\\
0 \text{ otherwise.}
\end{cases}
\end{equation}
\begin{equation}
{\phi_+^*}''(x)=
\begin{cases}
\frac{1}{2} \text{ if $x \geq -2$},\\
0 \text{ otherwise.}
\end{cases}
\end{equation}
\begin{equation}
f_*(x)+C = 2\frac{d\mu_c}{d\nu}(x)-2 \text{ almost everywhere with respect to } \mu_c
\end{equation}

\subsubsection{Reverse $\chi^2$ divergence}
\begin{equation}
\phi_+(x)= \begin{cases}
\frac{1}{x}+x-2 \text{ if $x \geq 0$,}\\
\infty \text{ otherwise.}
\end{cases}
\end{equation}
\begin{equation}
\partial \phi_+(x)= \begin{cases}
\left\{ 1-\frac{1}{x^2} \right\} \text{ if $x > 0$,}\\
\emptyset \text{ otherwise.}
\end{cases}
\end{equation}
\begin{equation}
\phi'(\infty)=1
\end{equation}
\begin{equation}
\phi_+^*(x)=\begin{cases}
2-2\sqrt{1-x} \text{ if $x \leq 1$,}\\
\infty \text{ otherwise.}
\end{cases}
\end{equation}
\begin{equation}
{\phi_+^*}'(x)=\frac{1}{\sqrt{1-x}}
\end{equation}
\begin{equation}
{\phi_+^*}''(x)=\frac{1}{2\sqrt{1-x}^3}
\end{equation}
\begin{equation}
f_*(x)+C = \begin{cases}
1-\frac{1}{\left(\frac{d\mu_c}{d\nu}(x)\right)^2} \text{ almost everywhere with respect to } \mu_c, \\
1 \text{ if } x \in \support(\mu_s)
\end{cases}
\end{equation}

\subsubsection{Squared Hellinger divergence}
\begin{equation}
\phi_+(x)= \begin{cases}
(\sqrt{x}-1)^2 \text{ if $x \geq 0$,}\\
\infty \text{ otherwise.}
\end{cases}
\end{equation}
\begin{equation}
\partial \phi_+(x)= \begin{cases}
\left\{ 1-\frac{1}{\sqrt{x}} \right\} \text{ if $x > 0$,}\\
\emptyset \text{ otherwise.}
\end{cases}
\end{equation}
\begin{equation}
\phi'(\infty)=1
\end{equation}
\begin{equation}
\phi_+^*(x)=\begin{cases}
\frac{x}{1-x} \text{ if $x \leq 1$,}\\
\infty \text{ otherwise.}
\end{cases}
\end{equation}
\begin{equation}
{\phi_+^*}'(x)=\frac{1}{(1-x)^2}
\end{equation}
\begin{equation}
{\phi_+^*}''(x)=\frac{2}{(1-x)^3}
\end{equation}
\begin{equation}
f_*(x)+C = \begin{cases}
1-\frac{1}{\sqrt{\frac{d\mu_c}{d\nu}(x)}} \text{ almost everywhere with respect to } \mu_c, \\
1 \text{ if } x \in \support(\mu_s)
\end{cases}
\end{equation}

\subsubsection{Jensen-Shannon divergence}
\begin{equation}
\phi_+(x)= \begin{cases}
x\log(x)-(x+1)\log(\frac{x+1}{2}) \text{ if $x \geq 0$,}\\
\infty \text{ otherwise.}
\end{cases}
\end{equation}
\begin{equation}
\partial \phi_+(x)= \begin{cases}
\left\{ \log(x)-\log(x+1)+\log(2) \right\} \text{ if $x > 0$,}\\
\emptyset \text{ otherwise.}
\end{cases}
\end{equation}
\begin{equation}
\phi'(\infty)=\log(2)
\end{equation}
\begin{equation}
\phi_+^*(x)=\begin{cases}
-\log(2-e^x) \text{ if $x \leq \log(2)$,}\\
\infty \text{ otherwise.}
\end{cases}
\end{equation}
\begin{equation}
{\phi_+^*}'(x)=\frac{1}{2e^{-x}-1}
\end{equation}
\begin{equation}
{\phi_+^*}''(x)=\frac{2e^x}{(e^x-2)^2}
\end{equation}
\begin{equation}
f_*(x)+C = \begin{cases}
\log\left(\frac{d\mu_c}{d\nu}(x)\right)-\log\left(\frac{d\mu_c}{d\nu}(x)+1\right)+\log(2) \text{ almost everywhere with respect to } \mu_c, \\
\log(2) \text{ if } x \in \support(\mu_s)
\end{cases}
\end{equation}

\subsubsection{Jeffreys divergence}
\begin{equation}
\phi_+(x)= \begin{cases}
(x-1)\log(x) \text{ if $x \geq 0$,}\\
\infty \text{ otherwise.}
\end{cases}
\end{equation}
\begin{equation}
\partial \phi_+(x)= \begin{cases}
\left\{ \log(x)-\frac{1}{x}+1 \right\} \text{ if $x > 0$,}\\
\emptyset \text{ otherwise.}
\end{cases}
\end{equation}
\begin{equation}
\phi'(\infty)=\infty
\end{equation}
\begin{equation}
\phi_+^*(x)=x+W(e^{1-x})+\frac{1}{W(e^{1-x})}-2
\end{equation}
\begin{equation}
{\phi_+^*}'(x)=\frac{1}{W(e^{1-x})}
\end{equation}
\begin{equation}
{\phi_+^*}''(x)=\frac{1}{W(e^{1-x})}-\frac{1}{W(e^{1-x})+1}
\end{equation}
\begin{equation}
f_*(x)+C = \log\left(\frac{d\mu_c}{d\nu}(x)\right)-\frac{1}{\frac{d\mu_c}{d\nu}(x)}+1 \text{ almost everywhere with respect to } \mu_c
\end{equation}
$W$ denotes the principal branch of the Lambert W function, also called the product logarithm, defined implicitly by the relation $W(x)e^{W(x)}=x$. Similarly to the proposed conjugates, it is computed by Newton's method and its gradient by the implicit function theorem.

\subsubsection{Triangular discrimination divergence}
\begin{equation}
\phi_+(x)= \begin{cases}
\frac{(x-1)^2}{x+1} \text{ if $x \geq 0$,}\\
\infty \text{ otherwise.}
\end{cases}
\end{equation}
\begin{equation}
\partial \phi_+(x)= \begin{cases}
\left\{ \frac{(x-1)(x+3)}{(x+1)^2} \right\} \text{ if $x \geq 0$,}\\
\emptyset \text{ otherwise.}
\end{cases}
\end{equation}
\begin{equation}
\phi'(\infty)=1
\end{equation}
\begin{equation}
\phi_+^*(x)=\begin{cases}
-1 \text{ if $x < -3$,}\\
(\sqrt{1-x}-1)(\sqrt{1-x}-3) \text{ if $-3 \leq x \leq 1$,}\\
\infty \text{ otherwise.}
\end{cases}
\end{equation}
\begin{equation}
{\phi_+^*}'(x)=\begin{cases}
0 \text{ if $x < -3$,}\\
\frac{2}{\sqrt{1-x}}-1 \text{ if $-3 \leq x \leq 1$}\\
\end{cases}
\end{equation}
\begin{equation}
{\phi_+^*}''(x)=\begin{cases}
0 \text{ if $x < -3$,}\\
\frac{1}{(\sqrt{1-x})^3} \text{ if $-3 \leq x \leq 1$}\\
\end{cases}
\end{equation}
\begin{equation}
f_*(x)+C = \begin{cases}
\frac{\left(\frac{d\mu_c}{d\nu}(x)-1\right)\left(\frac{d\mu_c}{d\nu}(x)+3\right)}{\left(\frac{d\mu_c}{d\nu}(x)+1\right)^2} \text{ almost everywhere with respect to } \mu_c, \\
1 \text{ if } x \in \support(\mu_s)
\end{cases}
\end{equation}

\subsection{Experiments} \label{appendix_experiments}

\subsubsection{MY$f$-GAN on CIFAR-10}
The implementation was done in TensorFlow based on the official codebase of \citet{Adleretal2018}, with the critic and generator architectures being faithful reimplementations of the residual architecture from \citet{Gulrajanietal2017}. We used both the train and test parts of the CIFAR-10 dataset with randomly flipping images and adding uniform noise as augmentation. Minibatch size was $128$, which for the critic included $64$ real and $64$ generated samples. We used the ADAM optimizer with parameters $\beta_1=0,\beta_2=0.9$ and constant learning rate $2\times10^{-4}$, and trained the model for $100000$ iterations with $5$ gradient descent step per iteration for the critic, and $1$ for the generator. We monitored performance by evaluating the Inception Score at every $1000$ iteration during training on $10000$ generated samples, and once at the end of training on $100000$ generated samples, as well as the FID on $50000$ generated and $50000$ real samples at the end of training. We applied exponential moving averaging to the generator weights $\theta_g$ with coefficient $0.9999$. The Lambert W function implementation was based on \href{https://github.com/jackd/lambertw}{https://github.com/jackd/lambertw}. Trainings were done on GeForce 2080Ti GPUs running at $200$-$250$W and took around $12$ hours to complete, leading to an estimated $2.4$-$3$kWh power consumption. Computing the conjugates via the proposed algorithm introduced a computational overhead that induced $15\%$ longer training time at worst (for the Jeffreys divergence) compared to closed-form conjugates. Due to the large number of hyperparameters, we ran each setting once. Generated images can be seen in Figure~\ref{generated_images_kl}, Figure~\ref{generated_images_rkl}, Figure~\ref{generated_images_chi2}, Figure~\ref{generated_images_rchi2}, Figure~\ref{generated_images_hellinger2}, Figure~\ref{generated_images_js}, Figure~\ref{generated_images_jeffreys}, Figure~\ref{generated_images_triangular}, Figure~\ref{generated_images_tv} and Figure~\ref{generated_images_trivial}, with missing images denoting failed training such as discriminator collapse or numerical instabilities.

\begin{figure}[ht]
\vskip 0.2in
\begin{center}
\subfigure[$\rightarrow \beta=0$]{
\includegraphics[width=.3\columnwidth]{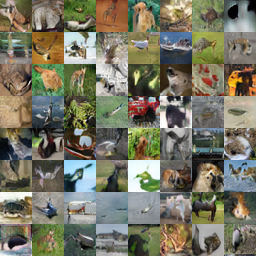}
}
\subfigure[$\rightarrow \alpha=1.05,\beta=1$]{
\includegraphics[width=.3\columnwidth]{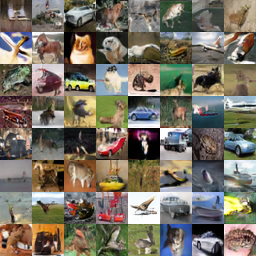}
}
\subfigure[$\rightarrow \alpha=2,\beta=1$]{
\includegraphics[width=.3\columnwidth]{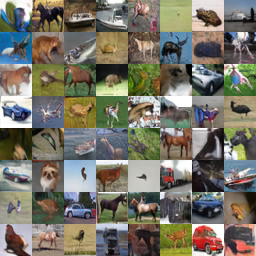}
}
\subfigure[$\leftarrow \alpha=1.05,\beta=1$]{
\includegraphics[width=.3\columnwidth]{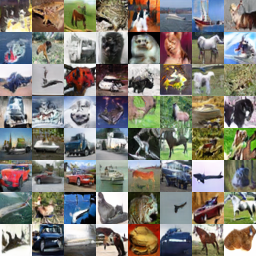}
}
\subfigure[$\leftarrow \alpha=2,\beta=1$]{
\includegraphics[width=.3\columnwidth]{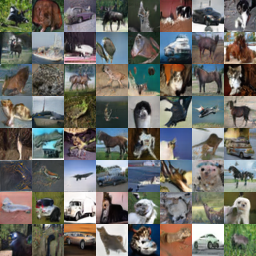}
}
\subfigure[$\leftarrow \alpha=\infty,\beta=0.5 \to 0.2$]{
\includegraphics[width=.3\columnwidth]{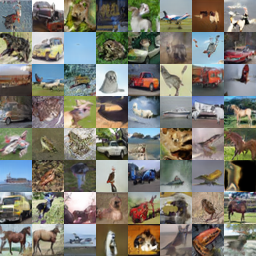}
}
\caption{MY$f$-GAN generated images with $D_\phi$ being the Kullback-Leibler divergence}
\label{generated_images_kl}
\end{center}
\vskip -0.2in
\end{figure}

\begin{figure}[ht]
\vskip 0.2in
\begin{center}
\subfigure[$\rightarrow \beta=0$]{
\includegraphics[width=.3\columnwidth]{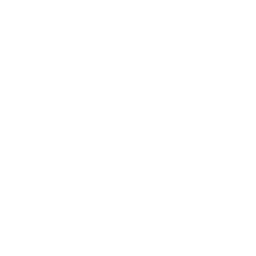}
}
\subfigure[$\rightarrow \alpha=1.05,\beta=1$]{
\includegraphics[width=.3\columnwidth]{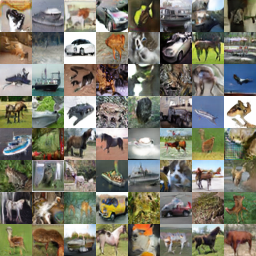}
}
\subfigure[$\rightarrow \alpha=2,\beta=1$]{
\includegraphics[width=.3\columnwidth]{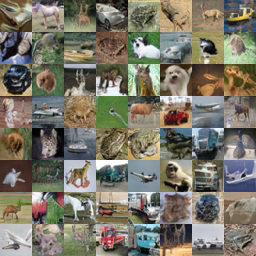}
}
\subfigure[$\leftarrow \alpha=1.05,\beta=1$]{
\includegraphics[width=.3\columnwidth]{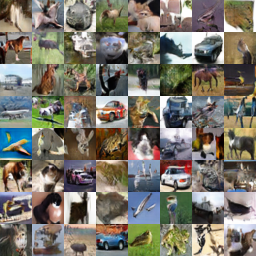}
}
\subfigure[$\leftarrow \alpha=2,\beta=1$]{
\includegraphics[width=.3\columnwidth]{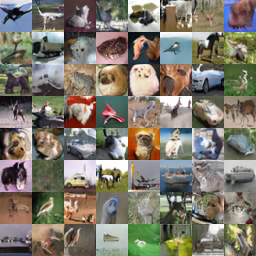}
}
\subfigure[$\leftarrow \alpha=\infty,\beta=0.5 \to 0.2$]{
\includegraphics[width=.3\columnwidth]{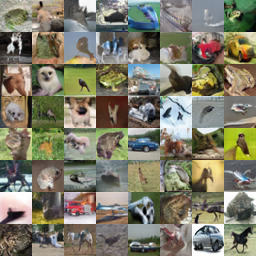}
}
\caption{MY$f$-GAN generated images with $D_\phi$ being the reverse Kullback-Leibler divergence}
\label{generated_images_rkl}
\end{center}
\vskip -0.2in
\end{figure}

\begin{figure}[ht]
\vskip 0.2in
\begin{center}
\subfigure[$\rightarrow \beta=0$]{
\includegraphics[width=.3\columnwidth]{img/missing}
}
\subfigure[$\rightarrow \alpha=1.05,\beta=1$]{
\includegraphics[width=.3\columnwidth]{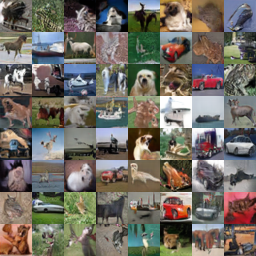}
}
\subfigure[$\rightarrow \alpha=2,\beta=1$]{
\includegraphics[width=.3\columnwidth]{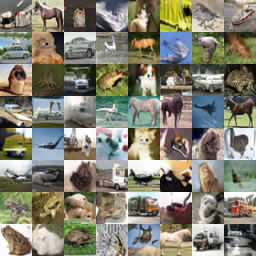}
}
\subfigure[$\leftarrow \alpha=1.05,\beta=1$]{
\includegraphics[width=.3\columnwidth]{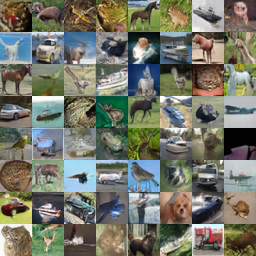}
}
\subfigure[$\leftarrow \alpha=2,\beta=1$]{
\includegraphics[width=.3\columnwidth]{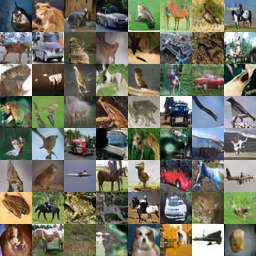}
}
\subfigure[$\leftarrow \alpha=\infty,\beta=0.5 \to 0.2$]{
\includegraphics[width=.3\columnwidth]{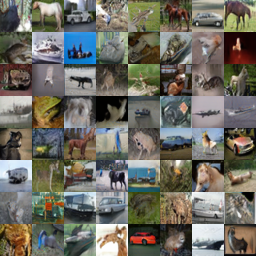}
}
\caption{MY$f$-GAN generated images with $D_\phi$ being the $\chi^2$ divergence}
\label{generated_images_chi2}
\end{center}
\vskip -0.2in
\end{figure}

\begin{figure}[ht]
\vskip 0.2in
\begin{center}
\subfigure[$\rightarrow \beta=0$]{
\includegraphics[width=.3\columnwidth]{img/missing}
}
\subfigure[$\rightarrow \alpha=1.05,\beta=1$]{
\includegraphics[width=.3\columnwidth]{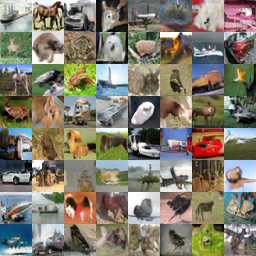}
}
\subfigure[$\rightarrow \alpha=2,\beta=1$]{
\includegraphics[width=.3\columnwidth]{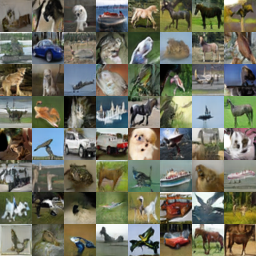}
}
\subfigure[$\leftarrow \alpha=1.05,\beta=1$]{
\includegraphics[width=.3\columnwidth]{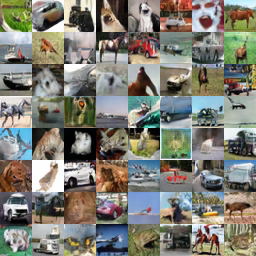}
}
\subfigure[$\leftarrow \alpha=2,\beta=1$]{
\includegraphics[width=.3\columnwidth]{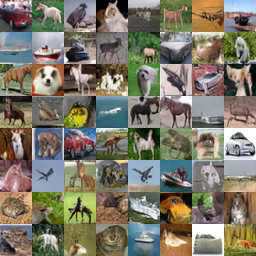}
}
\subfigure[$\leftarrow \alpha=\infty,\beta=0.5 \to 0.2$]{
\includegraphics[width=.3\columnwidth]{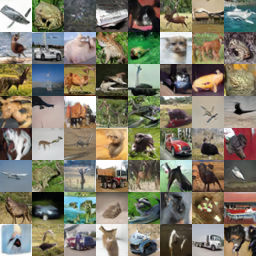}
}
\caption{MY$f$-GAN generated images with $D_\phi$ being the reverse $\chi^2$ divergence}
\label{generated_images_rchi2}
\end{center}
\vskip -0.2in
\end{figure}

\begin{figure}[ht]
\vskip 0.2in
\begin{center}
\subfigure[$\rightarrow \beta=0$]{
\includegraphics[width=.3\columnwidth]{img/missing}
}
\subfigure[$\rightarrow \alpha=1.05,\beta=1$]{
\includegraphics[width=.3\columnwidth]{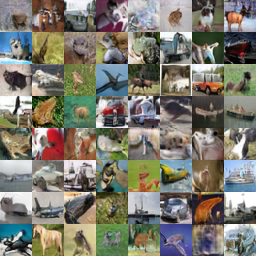}
}
\subfigure[$\rightarrow \alpha=2,\beta=1$]{
\includegraphics[width=.3\columnwidth]{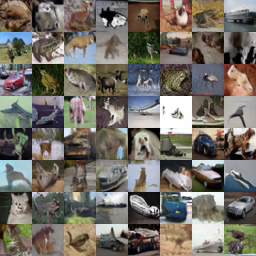}
}
\subfigure[$\leftarrow \alpha=1.05,\beta=1$]{
\includegraphics[width=.3\columnwidth]{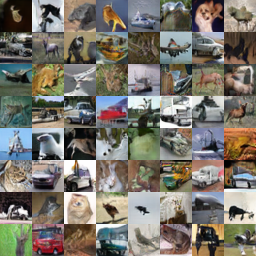}
}
\subfigure[$\leftarrow \alpha=2,\beta=1$]{
\includegraphics[width=.3\columnwidth]{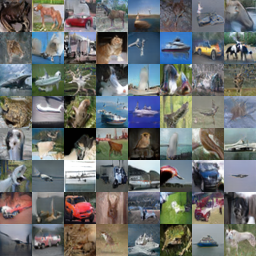}
}
\subfigure[$\leftarrow \alpha=\infty,\beta=0.5 \to 0.2$]{
\includegraphics[width=.3\columnwidth]{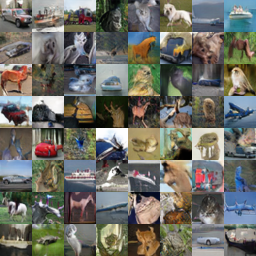}
}
\caption{MY$f$-GAN generated images with $D_\phi$ being the squared Hellinger divergence}
\label{generated_images_hellinger2}
\end{center}
\vskip -0.2in
\end{figure}

\begin{figure}[ht]
\vskip 0.2in
\begin{center}
\subfigure[$\rightarrow \beta=0$]{
\includegraphics[width=.3\columnwidth]{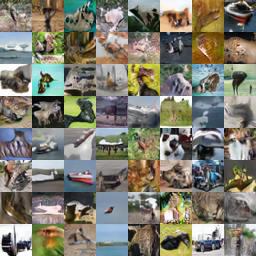}
}
\subfigure[$\rightarrow \alpha=1.05,\beta=1$]{
\includegraphics[width=.3\columnwidth]{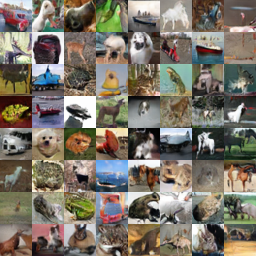}
}
\subfigure[$\rightarrow \alpha=2,\beta=1$]{
\includegraphics[width=.3\columnwidth]{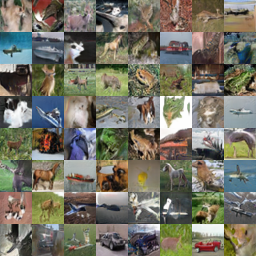}
}
\subfigure[$\leftarrow \alpha=1.05,\beta=1$]{
\includegraphics[width=.3\columnwidth]{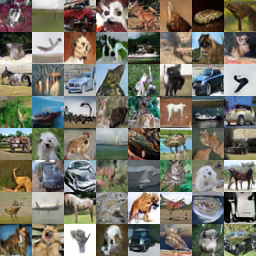}
}
\subfigure[$\leftarrow \alpha=2,\beta=1$]{
\includegraphics[width=.3\columnwidth]{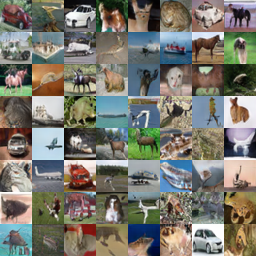}
}
\subfigure[$\leftarrow \alpha=\infty,\beta=0.5 \to 0.2$]{
\includegraphics[width=.3\columnwidth]{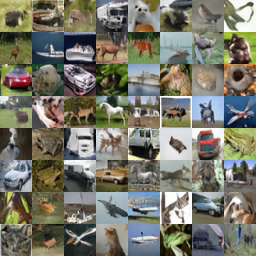}
}
\caption{MY$f$-GAN generated images with $D_\phi$ being the Jensen-Shannon divergence}
\label{generated_images_js}
\end{center}
\vskip -0.2in
\end{figure}

\begin{figure}[ht]
\vskip 0.2in
\begin{center}
\subfigure[$\rightarrow \beta=0$]{
\includegraphics[width=.3\columnwidth]{img/missing}
}
\subfigure[$\rightarrow \alpha=1.05,\beta=1$]{
\includegraphics[width=.3\columnwidth]{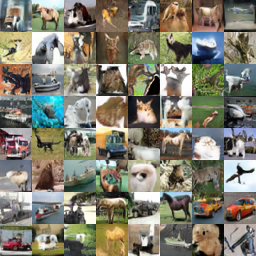}
}
\subfigure[$\rightarrow \alpha=2,\beta=1$]{
\includegraphics[width=.3\columnwidth]{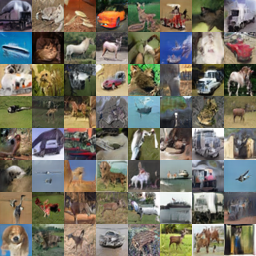}
}
\subfigure[$\leftarrow \alpha=1.05,\beta=1$]{
\includegraphics[width=.3\columnwidth]{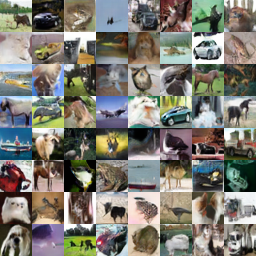}
}
\subfigure[$\leftarrow \alpha=2,\beta=1$]{
\includegraphics[width=.3\columnwidth]{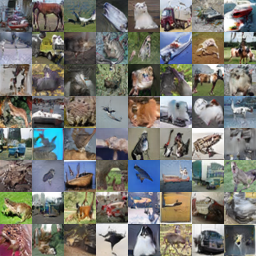}
}
\subfigure[$\leftarrow \alpha=\infty,\beta=0.5 \to 0.2$]{
\includegraphics[width=.3\columnwidth]{img/missing}
}
\caption{MY$f$-GAN generated images with $D_\phi$ being the Jeffreys divergence}
\label{generated_images_jeffreys}
\end{center}
\vskip -0.2in
\end{figure}

\begin{figure}[ht]
\vskip 0.2in
\begin{center}
\subfigure[$\rightarrow \beta=0$]{
\includegraphics[width=.3\columnwidth]{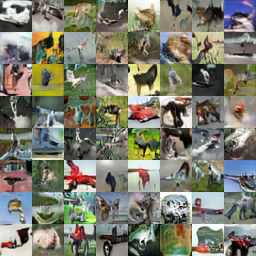}
}
\subfigure[$\rightarrow \alpha=1.05,\beta=1$]{
\includegraphics[width=.3\columnwidth]{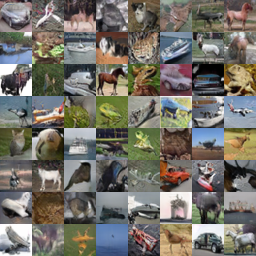}
}
\subfigure[$\rightarrow \alpha=2,\beta=1$]{
\includegraphics[width=.3\columnwidth]{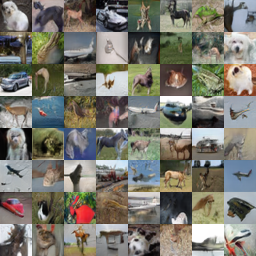}
}
\subfigure[$\leftarrow \alpha=1.05,\beta=1$]{
\includegraphics[width=.3\columnwidth]{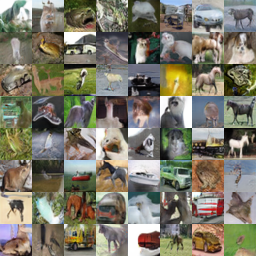}
}
\subfigure[$\leftarrow \alpha=2,\beta=1$]{
\includegraphics[width=.3\columnwidth]{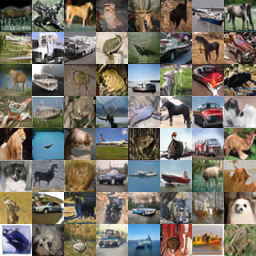}
}
\subfigure[$\leftarrow \alpha=\infty,\beta=0.5 \to 0.2$]{
\includegraphics[width=.3\columnwidth]{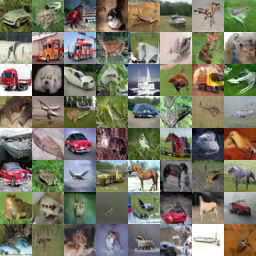}
}
\caption{MY$f$-GAN generated images with $D_\phi$ being the triangular discrimination divergence}
\label{generated_images_triangular}
\end{center}
\vskip -0.2in
\end{figure}

\begin{figure}[ht]
\vskip 0.2in
\begin{center}
\subfigure[$\rightarrow \beta=0$]{
\includegraphics[width=.3\columnwidth]{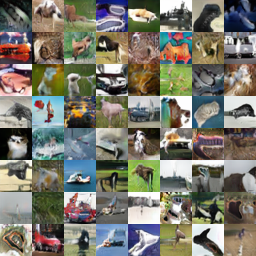}
}
\subfigure[$\rightarrow \alpha=1.05,\beta=1$]{
\includegraphics[width=.3\columnwidth]{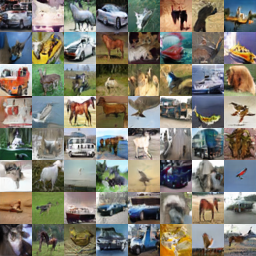}
}
\subfigure[$\rightarrow \alpha=2,\beta=1$]{
\includegraphics[width=.3\columnwidth]{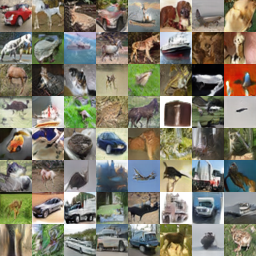}
}
\subfigure[$\leftarrow \alpha=1.05,\beta=1$]{
\includegraphics[width=.3\columnwidth]{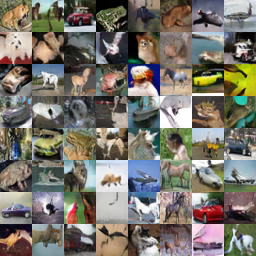}
}
\subfigure[$\leftarrow \alpha=2,\beta=1$]{
\includegraphics[width=.3\columnwidth]{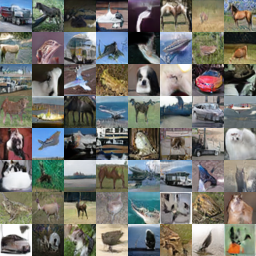}
}
\subfigure[$\leftarrow \alpha=\infty,\beta=0.5 \to 0.2$]{
\includegraphics[width=.3\columnwidth]{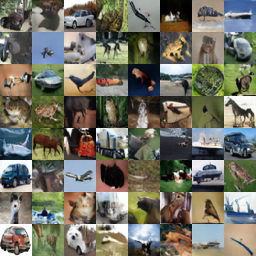}
}
\caption{MY$f$-GAN generated images with $D_\phi$ being the total variation divergence}
\label{generated_images_tv}
\end{center}
\vskip -0.2in
\end{figure}

\begin{figure}[ht]
\vskip 0.2in
\begin{center}
\subfigure[$\alpha=1.05,\beta=1$]{
\includegraphics[width=.3\columnwidth]{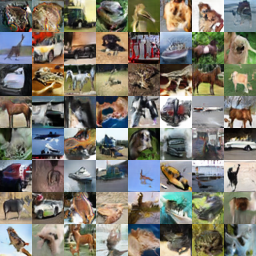}
}
\subfigure[$\alpha=2,\beta=1$]{
\includegraphics[width=.3\columnwidth]{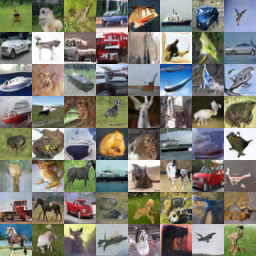}
}
\subfigure[$\alpha=\infty,\beta=0.5 \to 0.2$]{
\includegraphics[width=.3\columnwidth]{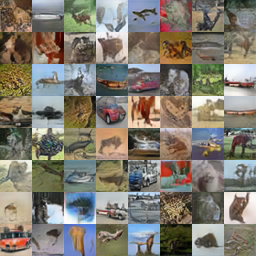}
}
\caption{MY$f$-GAN generated images with $D_\phi$ being the trivial divergence}
\label{generated_images_trivial}
\end{center}
\vskip -0.2in
\end{figure}

\subsubsection{1D Gaussian distributions}
For a pair $\mathcal{N}(\mu_1,\sigma_1), \mathcal{N}(\mu_2,\sigma_2)$ of 1-dimensional Gaussian distributions, the corresponding probability distribution functions are the Radon-Nikodym derivatives with respect to the Lebesgue measure, hence by the chain rule one has
\begin{equation}
\frac{d\mathcal{N}(\mu_1,\sigma_1)}{d\mathcal{N}(\mu_2,\sigma_2)}(x) = \frac{\sigma_2}{\sigma_1} e^{\frac{1}{2}\left(\left(\frac{x-\mu_2}{\sigma_2}\right)^2 - \left(\frac{x-\mu_1}{\sigma_1}\right)^2\right)},
\end{equation}
so that Csisz\'ar potentials can be calculated in closed form if $\phi$ is of Legendre type. We have implemented a toy example to demonstrate that the proposed algorithm for calculating the conjugates enables training neural networks to approximate $f$-divergences based on the tight variational representations in the sense that the trained neural network closely approximates the corresponding Csisz\'ar potential in the case of 1-dimensional Gaussian distributions. Results are visualized in Figure~\ref{gaussian_potentials}, showing the probability distribution functions of the two Gaussians, the exact Csisz\'ar potential and the output of the trained neural network. For the Jeffreys divergence, numerical instabilities prevented us from obtaining the desired result, so that only the exact Csisz\'ar potential is visualized. Close approximation of the exact Csisz\'ar potential is evident in areas of higher density. It must be emphasized that the neural network is not explicitly trained to approximate the Csisz\'ar potential, only implicitly, by maximizing the tight variational formula, so that this experiment could be seen as a validation of both the algorithm for computing the conjugates and of the characterization of Csisz\'ar potentials.

\begin{figure}[ht]
\vskip 0.2in
\begin{center}
\subfigure[Kullback-Leibler]{
\includegraphics[width=.3\columnwidth]{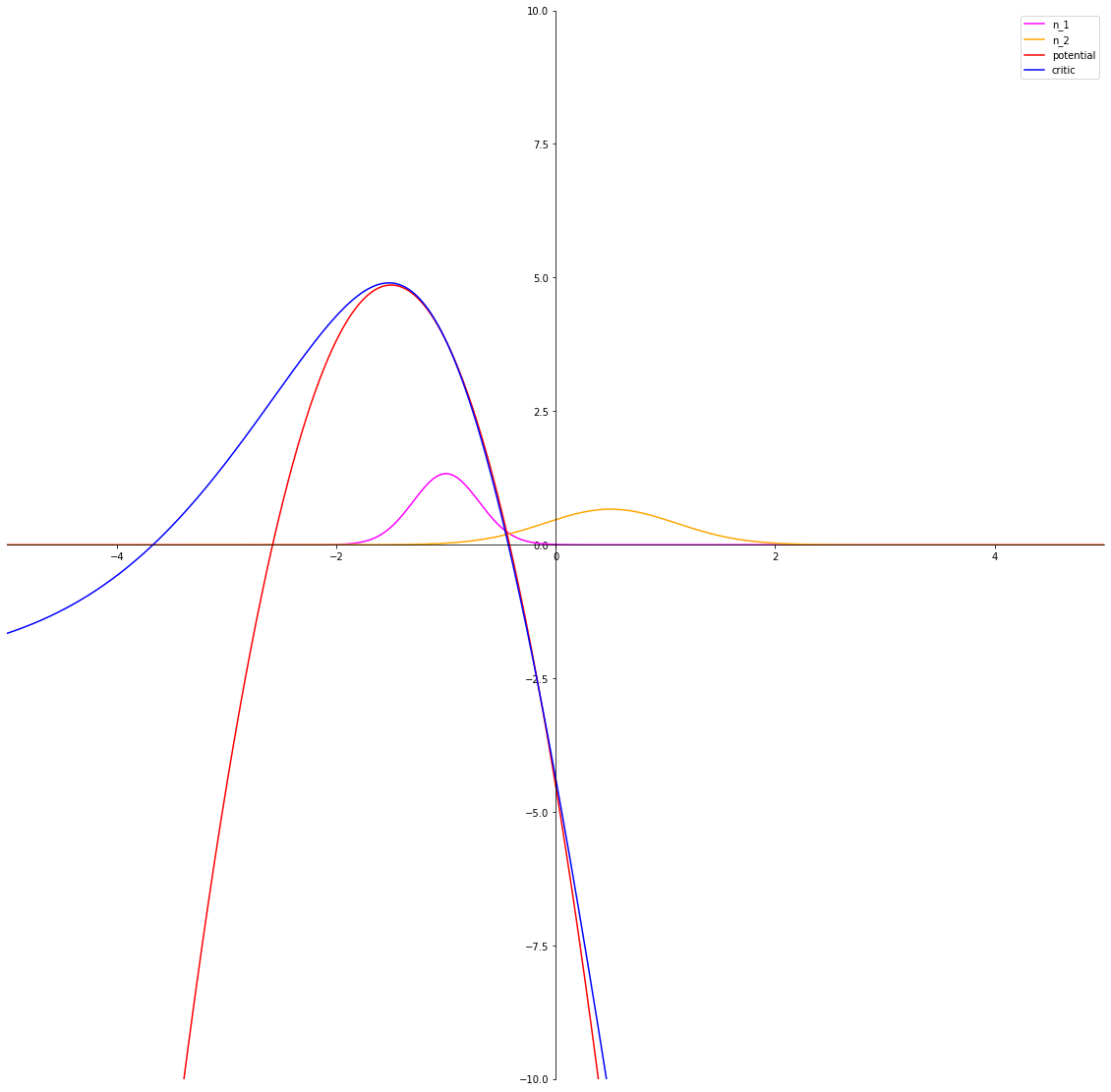}
}
\subfigure[Reverse Kullback-Leibler]{
\includegraphics[width=.3\columnwidth]{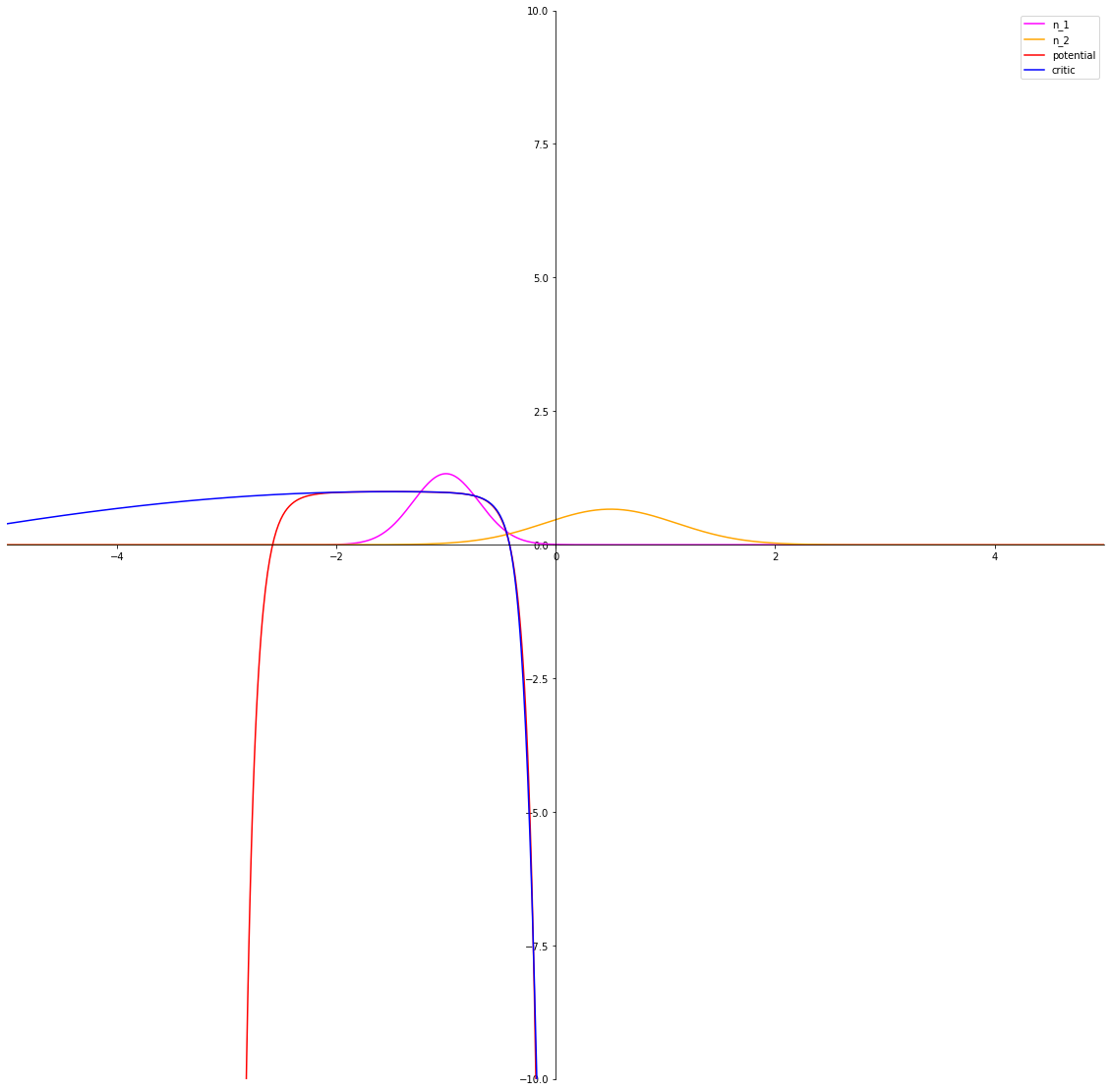}
}
\subfigure[$\chi^2$]{
\includegraphics[width=.3\columnwidth]{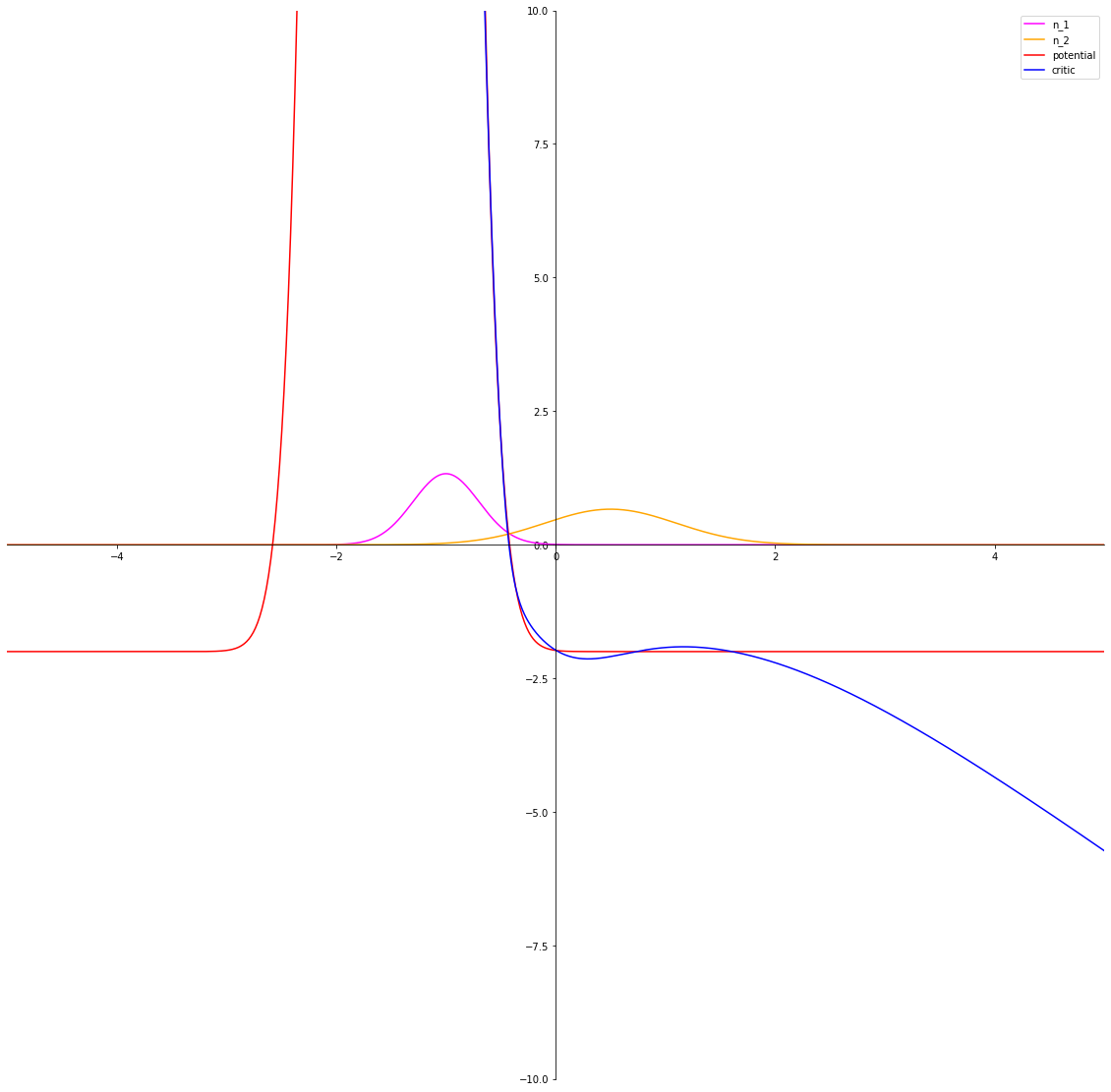}
}
\subfigure[Reverse $\chi^2$]{
\includegraphics[width=.3\columnwidth]{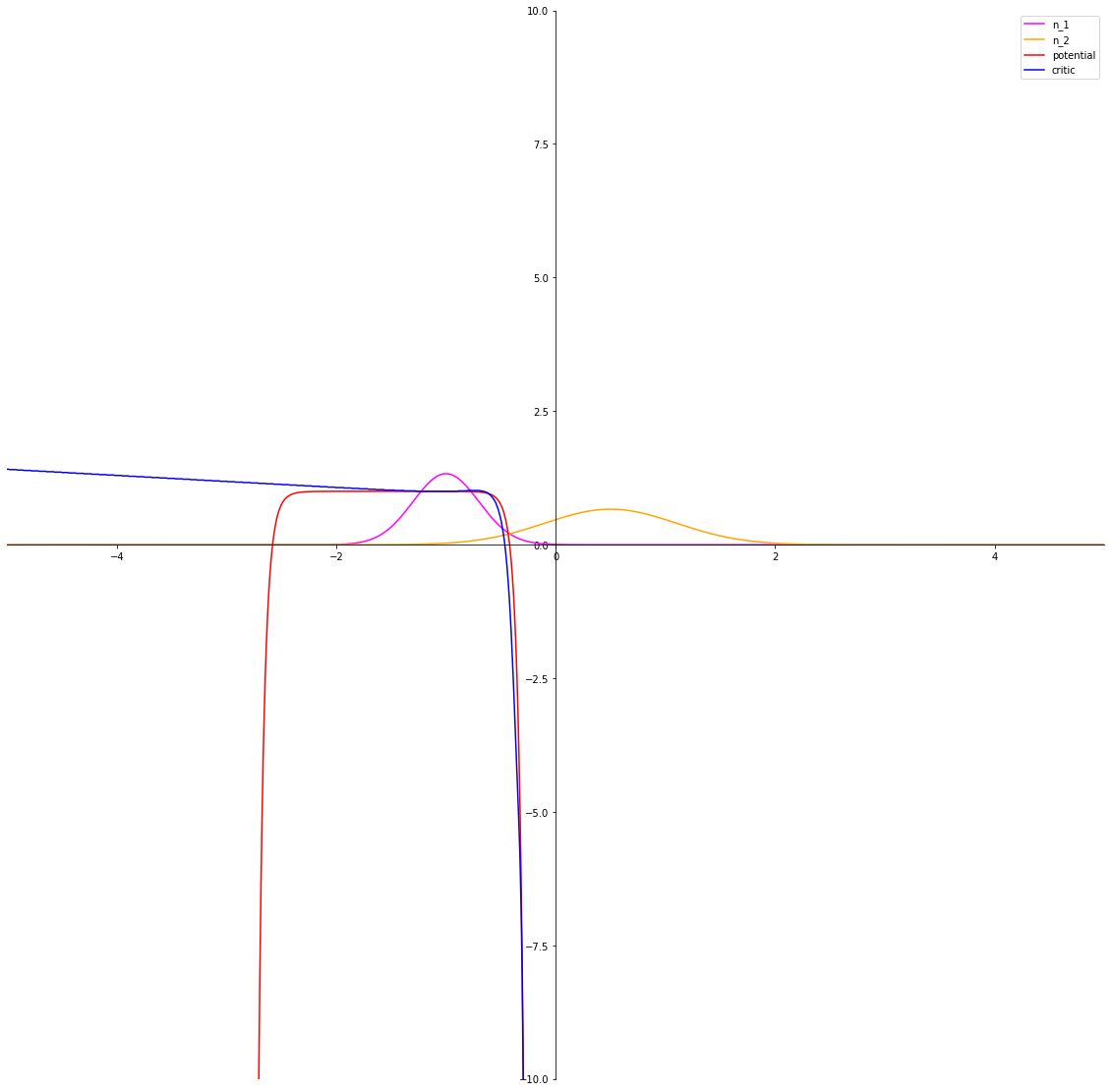}
}
\subfigure[Squared Hellinger]{
\includegraphics[width=.3\columnwidth]{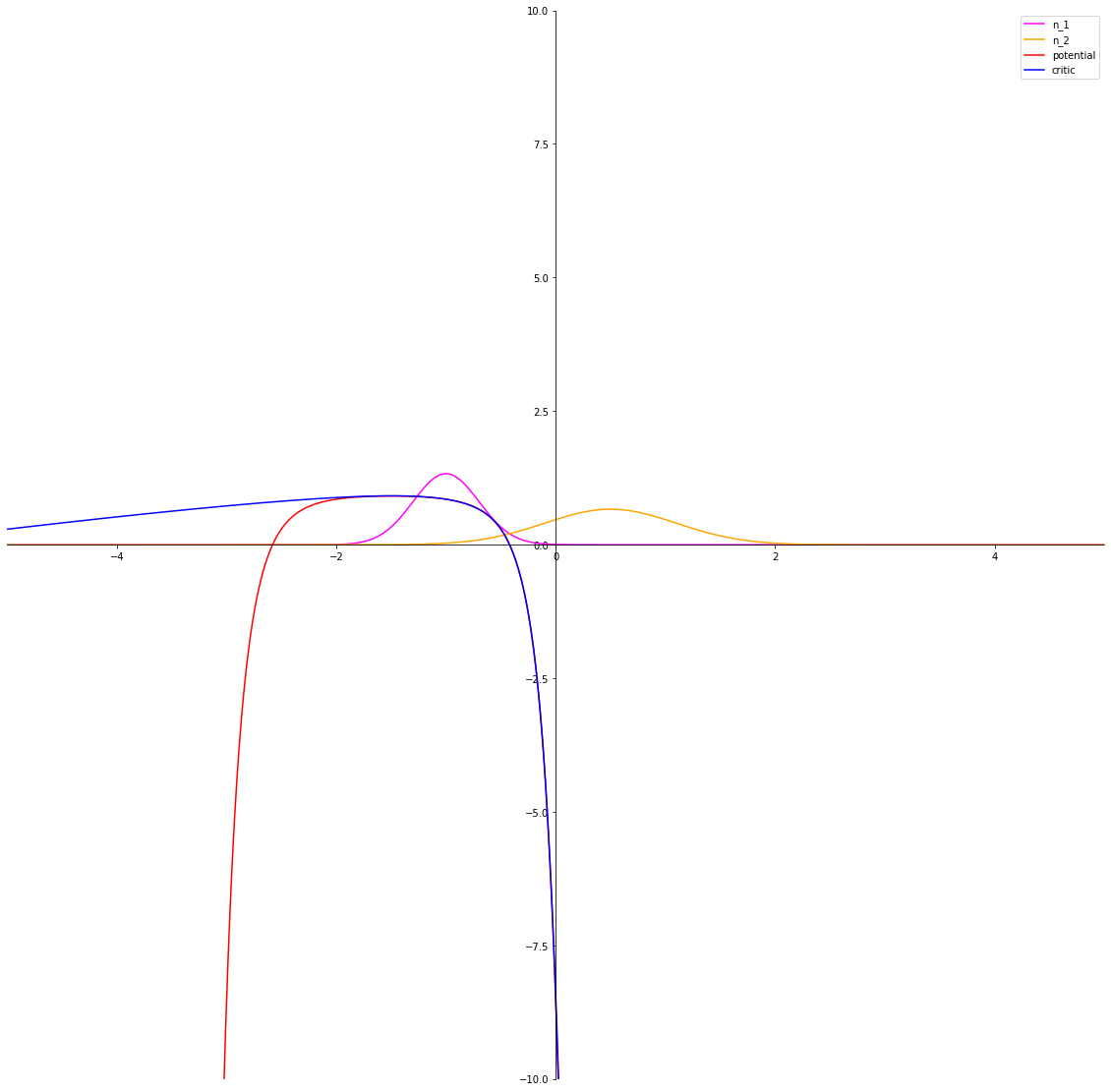}
}
\subfigure[Jensen-Shannon]{
\includegraphics[width=.3\columnwidth]{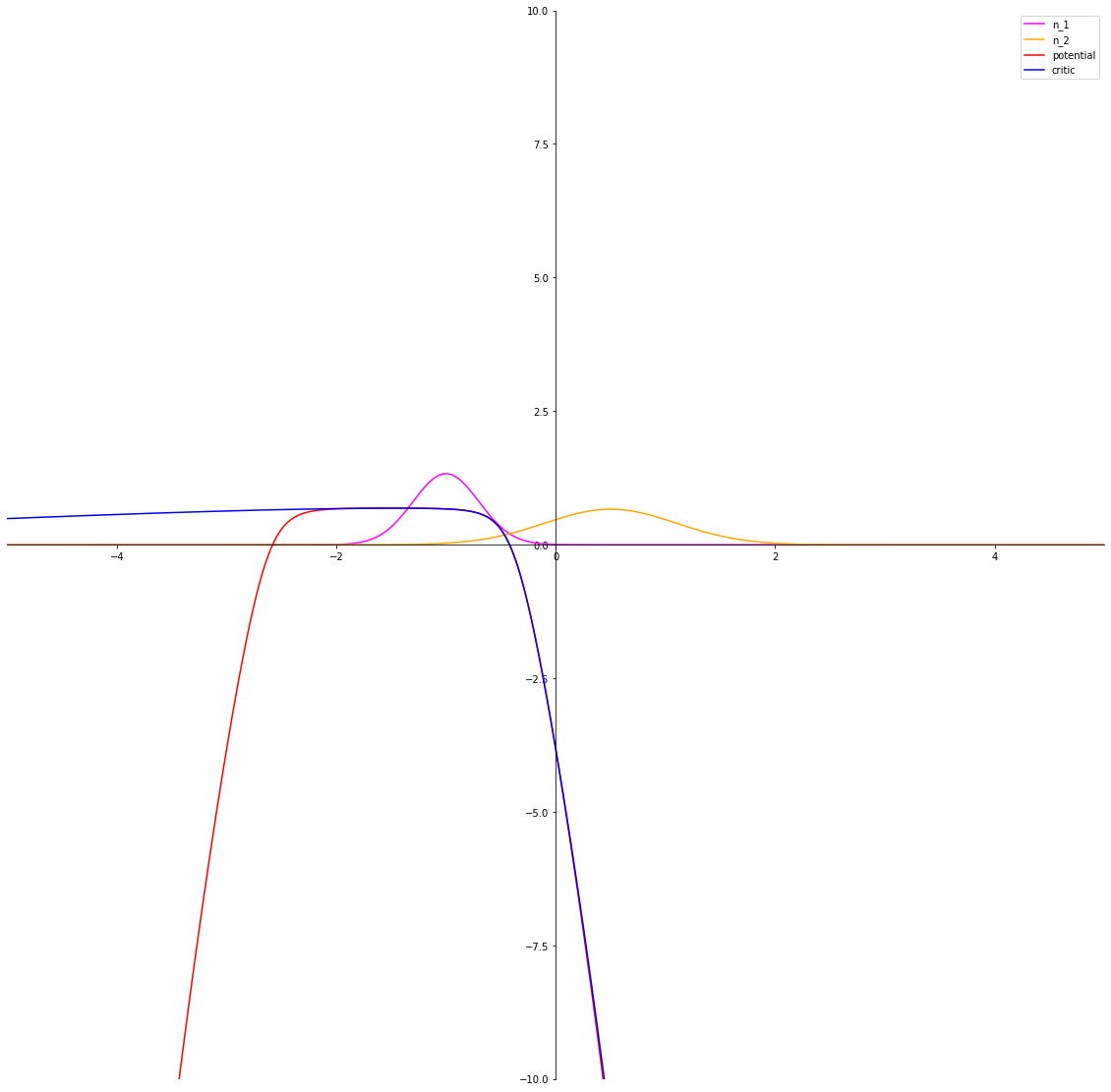}
}
\subfigure[Jeffreys]{
\includegraphics[width=.3\columnwidth]{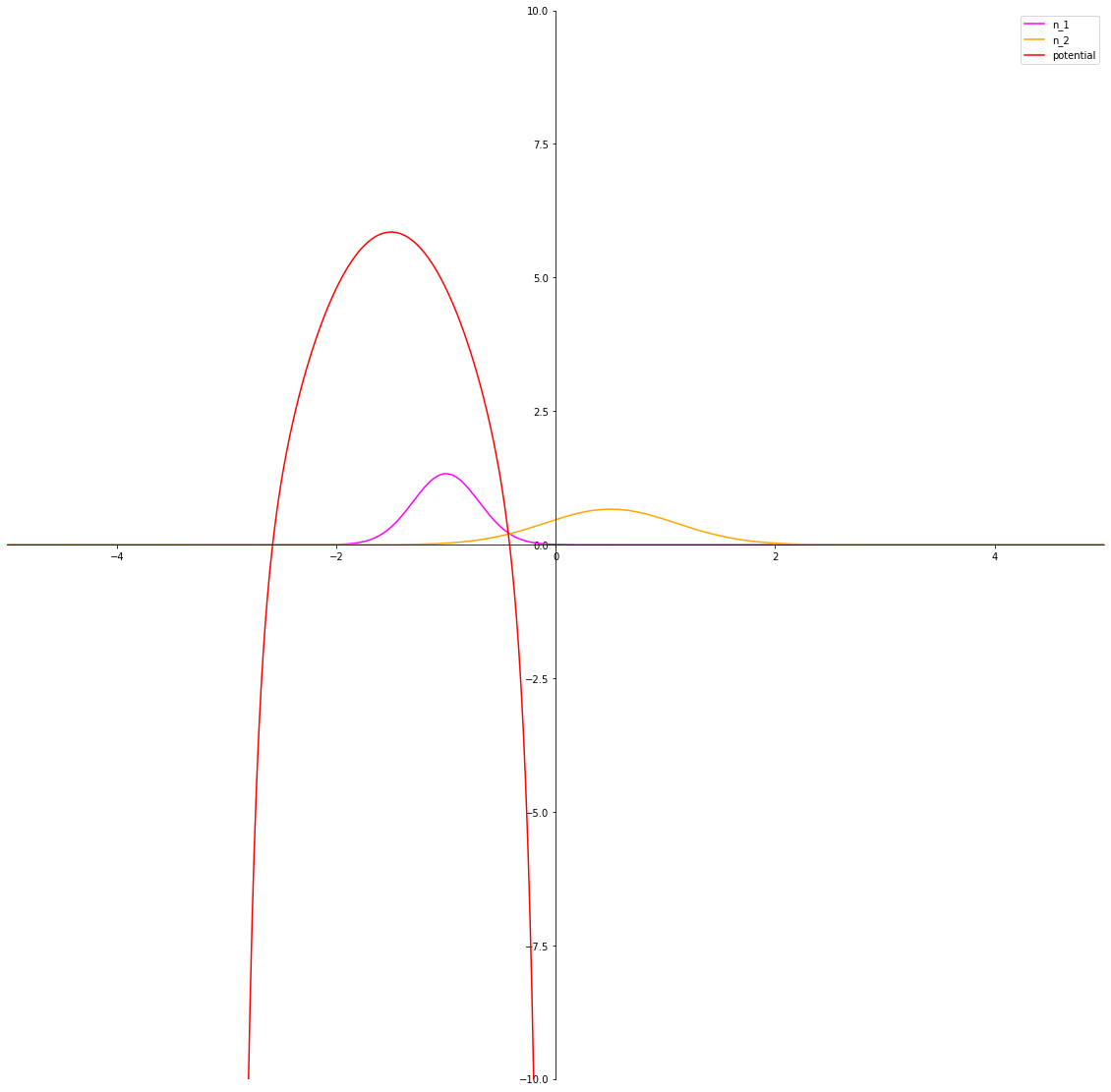}
}
\subfigure[Triangular discrimination]{
\includegraphics[width=.3\columnwidth]{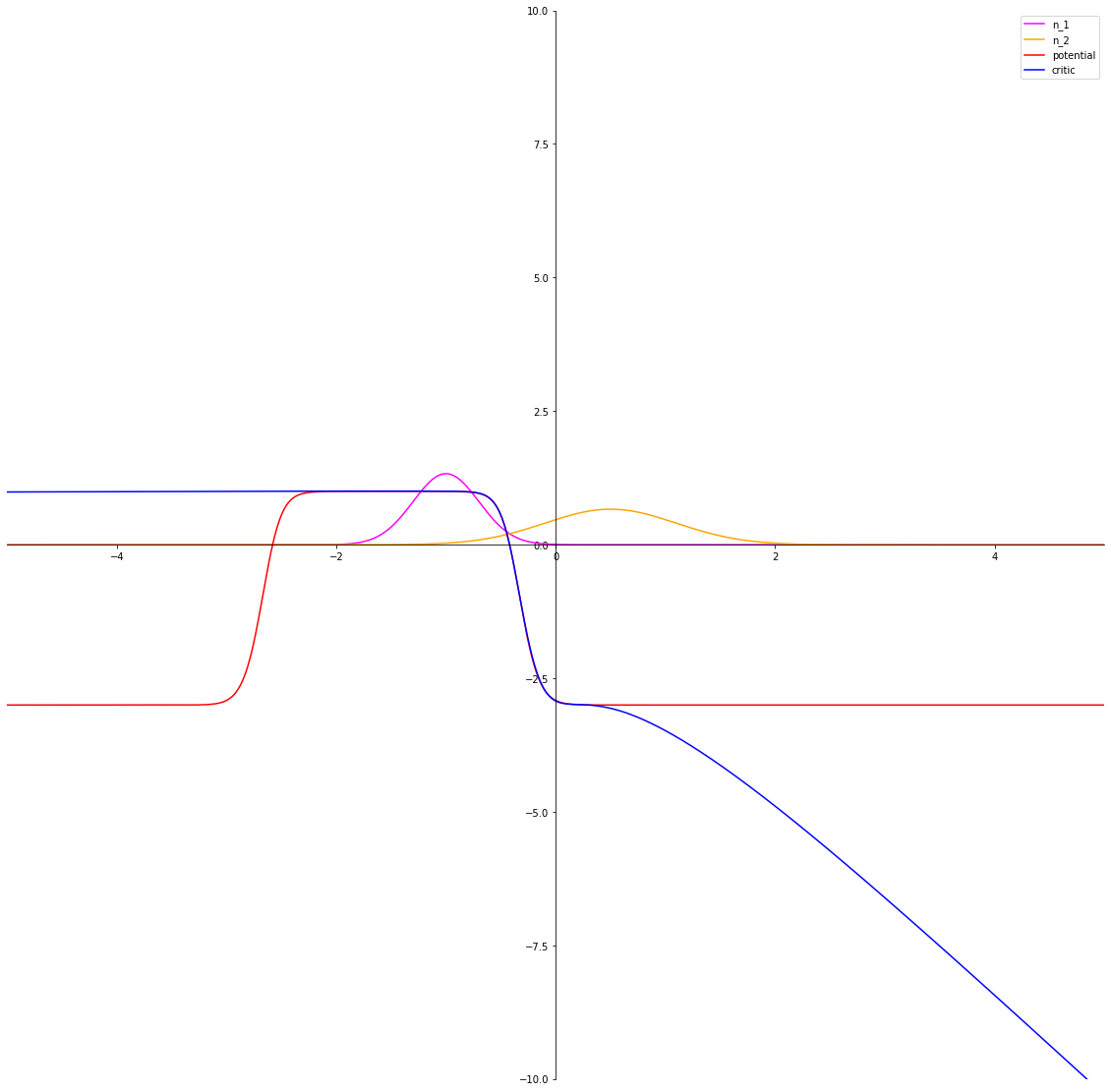}
}
\caption{Csisz\'ar potentials and trained critics of $\mathcal{N}(-1,0.3), \mathcal{N}(0.5,0.6)$}
\label{gaussian_potentials}
\end{center}
\vskip -0.2in
\end{figure}

\subsubsection{Categorical distributions}
For a pair $\mu,\nu$ of categorical distributions over an alphabet of size $n$, the potential $f$ can be considered an element of $\R^n$. We implemented a toy example to approximate $D_\phi$ in this case by optimizing an $f \in \R^n$ to maximize the formula inside the supremum in the tight variational representation, and found that the approximation accurately recovers the exact value of the divergence to at least $4$ decimals, except for the reverse $\chi^2$ divergence for which the approximation procedure is slightly less accurate. For the Kullback-Leibler and total variation divergences even though the conjugates are available in closed form, we implemented the proposed algorithm as well, and found that it works as well as the closed forms in this scenario. The generalized log-sum-exp trick is necessary in some cases to stabilize the algorithm.

\end{document}